\documentclass[a4paper, 11pt]{article}

\usepackage{fullpage}
\usepackage[utf8]{inputenc} % allow utf-8 input
\usepackage[T1]{fontenc}    % use 8-bit T1 fonts
\usepackage{url}            % simple URL typesetting
\usepackage{booktabs}       % professional-quality tables
\usepackage{amsfonts}
\usepackage{amsthm}% blackboard math symbols
\usepackage{amsmath,amssymb}
\usepackage{circledsteps}
\newcommand{\vertiii}[1]{{\left\vert\kern-0.25ex\left\vert\kern-0.25ex\left\vert #1 \right\vert\kern-0.25ex\right\vert\kern-0.25ex\right\vert}}

\DeclareMathOperator*{\argmin}{arg\,min}
\DeclareMathOperator{\Tr}{Tr}
\DeclareMathOperator*{\mvec}{\text{vec}}
\DeclareMathOperator{\ellinf}{\ell_{\infty}}
\usepackage{mathtools}
\newcommand{\defeq}{\vcentcolon=}
\theoremstyle{definition}
\newtheorem{definition}{Definition}
\theoremstyle{theorem}
\newtheorem{theorem}{Theorem}
\theoremstyle{corollary}
\newtheorem{corollary}{Corollary}
\theoremstyle{lemma}
\newtheorem{lemma}{Lemma}
\theoremstyle{remark}
\newtheorem{remark}{Remark}
\usepackage{nicefrac}       % compact symbols for 1/2, etc.
\usepackage{microtype}      % microtypography
\usepackage{xcolor}         % colors
\usepackage[export]{adjustbox}
\usepackage{caption}
\usepackage{subcaption}
\usepackage{enumitem}
\usepackage{chngcntr}
\usepackage{apptools}
\newcommand{\transpose}{\mathsf{T}}
\usepackage[ colorlinks = true,
linkcolor = blue,
urlcolor = blue,
citecolor = red,
anchorcolor = green,
]{hyperref}

\newcommand*\xbar[1]{%
  \hbox{%
    \vbox{%
      \hrule height 0.5pt % The actual bar
      \kern0.4ex%         % Distance between bar and symbol
      \hbox{%
        \kern-0.15em%      % Shortening on the left side
        \ensuremath{#1}%
        \kern-0.15em%      % Shortening on the right side
      }%
    }%
  }%
} 

\usepackage{natbib}
\setcitestyle{sort&compress}
\setcitestyle{numbers}
\setcitestyle{square}
\setcitestyle{comma}

\title{Learning the Structure of Large Networked Systems Obeying Conservation Laws \thanks{This work was supported in part by the National Science Foundation (NSF) under the grants CCF-2048223 and OAC-1934766, and by the National Institutes of Health (NIH) under the grant 1R01GM140468-01.}}

\date{}

\author{%
  Anirudh Rayas \\ Arizona State University\\ \texttt{ahrayas@asu.edu} \and Rajasekhar Anguluri\\ Arizona State Univeristy \\ \texttt{rangulur@asu.edu} \and Gautam Dasarathy\\
  Arizona State University\\
  \texttt{gautamd@asu.edu} \\
}

\begin{document}

\maketitle

\begin{abstract}
Many networked systems such as electric networks, the brain, and social networks of opinion dynamics are known to obey conservation laws. Examples of this phenomenon include the Kirchoff laws in electric networks and opinion consensus in social networks. Conservation laws in networked systems may be modeled as {\em balance equations} of the form $X = B^\ast Y$, where the sparsity pattern of $B^\ast \in \mathbb{R}^{p\times p}$ captures the connectivity of the network on $p$ nodes, and  $Y, X \in \mathbb{R}^p$ are vectors of ``potentials'' and ``injected flows'' at the nodes respectively. The node potentials $Y$ cause flows across edges  and the  flows $X$ injected at the nodes are extraneous to the network dynamics. In several practical systems, the network structure is often unknown and needs to be estimated from data to facilitate modeling, management, and control. To this end, one has access to samples of the node potentials $Y$, but only the statistics of the node injections $X$. Motivated by this important problem, we study the estimation of the sparsity structure of the matrix $B^{*}$ from $n$ samples of $Y$ under the assumption that the node injections $X$ follow a Gaussian distribution with a known covariance $\Sigma_X$. We propose a new $\ell_{1}$-regularized maximum likelihood estimator for tackling this problem in the high-dimensional regime where the size of the network may vastly be larger than the number of samples $n$. We show that this optimization problem is convex in the objective and admits a unique solution. Under a new mutual incoherence condition, we establish sufficient conditions on the triple $(n,p,d)$ for which exact sparsity recovery of $B^{*}$ is possible with high probability; $d$ is the degree of the underlying graph. We also establish guarantees for the recovery of $B^*$ in the element-wise maximum, Frobenius, and operator norms. Finally, we complement these theoretical results with experimental validation of the performance of the proposed estimator on synthetic and real-world data.
\end{abstract} 

\section{Introduction}
\label{sec:intro}
Let $\mathcal{G}=([p],E)$ be a directed graph on the vertex set $[p] \triangleq \{1,2,\ldots, p\}$ with a size $m$ edge set $E \subset [p]\times[p]$. Let $\mathcal{D}$ denote the $p\times m$ incidence matrix that encodes the edges of $\mathcal{G}$ as follows: each column of $\mathcal{D}$ is associated with an edge $(i,j)\in E$ and is a vector of zeros except at the locations $i$ and $j$ where it is $-1$ and $+1$ respectively. 
 
Let $X\in \mathbb{R}^p$ be a vector of \emph{injected flows or signals} at the vertices and  let $f\in \mathbb{R}^m$ be the vector of \emph{flows} through the edges. Then, the basic \emph{conservation law} between these flows may be expressed as $\mathcal{D}f+X=0$; that is, at each vertex, the flow (which is a linear combination of flows at the edges incident on the vertex) has to balance the injections. 
In physical systems, edge flows $f$ often arise as a way to balance the differences between certain \emph{potentials} $Y\in \mathbb{R}^{p}$ at the vertices. That is,  the flows  satisfy $f=-\mathcal{D}^\transpose Y$; notice that this implies that the flow at the edge $(i,j)$ is given by $Y_j - Y_i$. Thus, the above conservation law yields the following relationship, called a {\em balance equation}, between the node potentials and injected flows: 
\begin{align}\label{eq: flow systems}
    X-B^*Y = 0, 
\end{align}
 where  $B^*\triangleq \mathcal{D}\mathcal{D}^\transpose \in \mathbb{R}^{p\times p}$ is the symmetric Laplacian matrix \cite{bapat2010graphs, van2017modeling}. In an electrical circuit (with unit resistances on edges), $Y$ corresponds to the voltage potentials at the vertices, $f$ corresponds to the edge currents, and $X$ denotes the injected currents at the vertices. Indeed, this picture can be generalized by assigning weights to the edges of the network (conductances in the case of an electric network), and allowing the flows to be weighted by these weights. The model in \eqref{eq: flow systems} is referred to as generalized Kirchoff's law, and importantly, this models the relationship between flows and potentials in a range of systems that satisfy conservation laws such as hydraulic networks, opinion consensus in social networks, and transportation/distribution networks (see \cite{van2017modeling, mesbahi2010graph, temkin2020chemical} and references therein). 

It can be readily seen that the Laplacian $B^*$ is a positive semi-definite that encodes the edges of $\mathcal{G}$. Specifically, $(i,j)\in E$ if and only if $B^*_{ij}\ne 0$. The Laplacian lies at the heart of spectral graph theory~\citep{chung1997spectral}, and owing to its fascinating properties has found a range of applications in diverse areas such as image processing, manifold learning, spectral clustering, and bandits~\citep{shi2000normalized, belkin2005towards, von2007tutorial, valko2014spectral, lejeune2020thresholding}.  
In this paper, we consider a situation where the edge set $E$ of the graph is unknown and needs to be estimated from measurements of the node potentials $Y$. 
Based on the above discussion, we will cast this as a problem of learning an unknown positive definite $B^\ast$ (or the sparsity pattern thereof) from measurements of $Y$. Further, we suppose that we only have access to the statistics of $X$, namely, that it is a 0-mean Gaussian random vector with a covariance matrix $\Sigma_X$. The situations where $\Sigma_X$ is unknown and $B^\ast$ is non-invertible is briefly discussed in the remarks in Section~\ref{sec: problem-setup}. 
We list a variety of applications where this learning problem arises naturally. 

\begin{enumerate}
\item \emph{Topology learning in electric networks}: Consider an electric network (or circuit) with $p$ nodes, current injections 
$X$, node voltages $Y$, conductances $A_{ij}\geq 0$ between nodes $i$ and $j$, and shunt conductances $A_{ii}\geq 0$ connecting $i$-th node to the ground. The current-balance equation is given by \eqref{eq: flow systems}, where 
$B^*$ is the Laplacian with $B^*_{ij}=-A_{ij}$ and $B^*_{ii}=A_{ii}+\sum_{j=1}^nA_{ij}$ \citep{dorfler2012kron}. In real electric grids, current injections are unknown random variables. To ensure reliable power supply, learning $B^*$ and its underlying graph from  voltage samples is important and has been widely studied \citep{DekaTSG2020, li2020learning, GC-VK-LZ-GBG:21, anguluri2021grid}. The current-balance equation also appears in Markov chains and flow networks where Kirchoff laws apply \citep{van2017modeling, pozrikidis2014introduction}. 
%\rajmargin{comment on infection statistics COVID type: \citep{tomovski2015network, tomovski2021discrete}} 
\item \emph{Brain connectivity from graph filters}: The structural connectivity of the human brain is often studied using a network with nodes representing brain regions, and the edge weights representing the density of anatomical connections \citep{pongrattanakul2013sparse, hagmann2008mapping}. Recent studies showed that the weights can be inferred using \emph{graph filters} satisfying \eqref{eq: flow systems} with $B^*=(\sum_{l=0}^{L-1}h_lA^L)^{-1}$, where $A$ is the symmetric adjacency matrix; $h_l$ is the filter coefficient; and $X$ is the latent graph signal. For brain networks, \citep{shafipour2017network, marques2017stationary} showed that $L=3$ and $B^*=(I+\alpha A)^{-1}$ are reasonable. Graph filters are also used in social and protein interaction networks \citep{shafipour2019online, mateos2019connecting}.
\item \emph{Structural equation models (SEM)}: Structural equation models are used to explain relationships among exploratory variables in several domains; for e.g., psychoanalysis \citep{epskamp2018gaussian}, social sciences \citep{civelek2018essentials}, medical research, and neuroimaging \citep{beran2010structural, pruttiakaravanich2020convex, mclntosh1994structural}. Using SEMs, \citep{mogensen2021equality}
provided a causal interpretation of Linear Hawkes Processes. In SEM with no latent variables, we let $y=Ay+x$, where $A$ is the \emph{path matrix}. Then the SEM satisifies \eqref{eq: flow systems} with $B^*=I-A$.
\item \emph{Linear dynamical (diffusion) networks}: These network dynamics are satisfied by many systems including consensus dynamics, thermal capacitance networks, power swing dynamics,  \citep{talukdar2020physics}. Further, by lifting approach,  these dynamics can be used to study periodic/cyclic behavior in atmospheric systems \citep{talukdar2015reconstruction}. 
\end{enumerate}

Before we detail our topology discovery method, we comment on a few competing approaches that only have limited utility in our setting. First, penalized (nodal) regression methods \cite{meinshausen2006high} are not applicable here since these require samples of 
both $X$ and $Y$. Second, a recent line of work \cite{DekaTSG2020, anguluri2021grid} proposed estimating $B^\ast$ by estimating the inverse covariance (or precision) matrix $\Theta^*$ of $Y$ 
using the graphical LASSO (GLASSO)~\cite{yuan2007model, friedman2008sparse}. In particular, \cite{DekaTSG2020, anguluri2021grid} showed that $\Theta^*$ has non-zeros corresponding to those pairs of vertices that are connected by paths of length at most two; that is, the $(i,j)$-th entry of $\Theta^*$ is non zero if and only if $(i,j)$ is an edge in $\mathcal{G}$ or there is a $k \in [p]$ such that $i-k-j$ is a path in $\mathcal{G}$. The authors then estimated edges of $\mathcal{G}$ by identifying (and eliminating) the pairs of vertices that have two-hop connections in $\mathcal{G}$ (see Fig.~\ref{fig:motivation_figure})---for future reference, we call this estimator as GLASSO+2HR (hop refinement). However, this estimator requires strong structural assumptions on $\mathcal{G}$ such as triangle-freeness. Further, the precision matrix $\Theta^*$ of $Y$ is far more dense than the underlying graph $\mathcal{G}$ since $\Theta^*=(B^*)^2$; this results in sub-optimal data requirements for reliable recovery (see Remark \ref{rmk: GLASSO comparison}). Finally, if $\Sigma_X$ is a diagonal matrix, we can estimate the sparsity pattern of $B^*$ by taking the principal square root of the empirical covariance matrix of $Y$.  Unfortunately, this method does not allow for any correlation between the node injections which is not the case in practice. Moreover, this method is numerically unstable unless one has a large number of samples ($n$) of $Y$ so that the empirical covariance matrix is invertible, a requirement that is at odds with the high-dimensional regime where one typically desires $n$ to be smaller than the number of variables $p$.

In light of the limitations of previous approaches, we study a natural penalized maximum likelihood estimator for $B^*$ using the samples of $\{Y_i\}_{i=1}^n$. The advantage of this estimator is that it is not only statistically efficient but also obviates the assumptions imposed by the aforementioned methods. We now summarize the main contributions of the paper: 
\begin{itemize}[leftmargin=0.45cm]
    \item We propose a novel $\ell_1$ regularized maximum likelihood estimator (MLE) for $B^\ast$ from samples of $Y$. It is worth noting that the optimization program we propose is not the standard graphical LASSO program~\citep{yuan2007model, friedman2008sparse} as it involves terms that are quadratic in the optimization variable. Our first result shows that, notwithstanding its form, the $\ell_{1}$ regularized MLE is convex in $B$ and it has a unique minimum even in the high-dimensional regime ($n\ll p$) under certain standard conditions. 
    \item Under a new mutual incoherence condition, our second result provides a sufficient condition on the number of samples $n$ required to recover the exact sparsity of $B^{*}$ with high probability. Furthermore, under these sufficient conditions we also establish the consistency of our estimator in the element-wise maximum, Frobenius, and spectral norms. More precisely, we show that if $n = \Omega(d^{2}\log p)$ then with high probability $\Vert\widehat{B} - B^{*}\Vert_{\infty} \in \mathcal{O}(\sqrt{{\log p}/n})$.
    \item Finally, we complement our theoretical results with experimental results both on the synthetic data sets and data from a benchmark power distribution system. Our experiments demonstrate the clear benefit of the proposed estimator over baseline and competing methods.  

\end{itemize}
	\begin{figure*}[h]
		\centering
		% \fbox{\rule[-.5cm]{0cm}{4cm} \rule[-.5cm]{4cm}{0cm}}
		\includegraphics[width=1\textwidth]{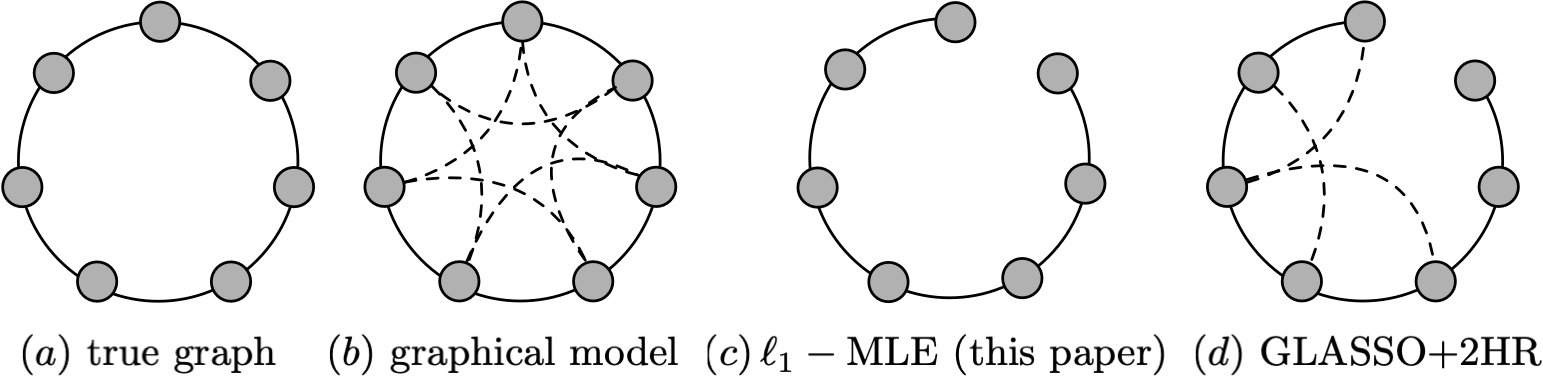}
		\caption{\small{Stylistic visualization of $\ell_1$-MLE vs GLASSO+2HR in \cite{DekaTSG2020}. (a) Graph $\mathcal{G}$ with $B^*$; (b) graphical model of $\Theta^*=B^*\Sigma^{-1}_X(B^*)^\transpose$; (c) estimate of $B^*$ by our proposed estimator; and (d) estimate of $B^*$ by GLASSO+2HR. Graph $\mathcal{G}$ and the graphical model of $\Theta^*$ have same set of vertices; however, in the latter, there are spurious edges (dashed lines in (b)) between vertices that are two-hop neighbors in $\mathcal{G}$ (see main text in Introduction). Consequently, GLASSO+2HR estimate, which relies on a estimate of $\Theta^*$ has also spurious edges. Instead, our $\ell_1$-MLE directly estimates $B^*$, and hence, there are no spuruious edges in it.}}
		\label{fig:motivation_figure}	
	\end{figure*} 

\noindent\textbf{Organization of the paper}: In Section \ref{sec: problem-setup}, we introduce an $\ell_1$-regularized ML estimation problem for networked systems obeying conservation laws. In Section \ref{sec: Statistical guarantees}, we show that this optimization problem is convex in the objective and establish consistency and support recovery rates for our estimator. In Section \ref{sec: experiments}, we provide simulation results. In Section \ref{sec: discussion}, we summarize our paper with future directions. 

\noindent\textbf{Notation}: For any two subsets $T_1$ and $T_2$ of $[p]\times [p]$, we denote by $A_{T_1T_2}$, the submatrix of $A$ with rows and columns indexed by $T_1$ and $T_2$, respectively. When $T_1=T_2$ we denote the submatrix by $A_{T_1}$. For a matrix $A=(A_{i,j})\in \mathbb{R}^{p\times p}$, we use $\Vert A\Vert_{\infty} \triangleq \max_{i,j}\vert A_{ij}\vert $ to denote the maximum element-wise norm, and $\|A\|_F$ and $\Vert A\Vert_{2}$ to denote the Frobenius norm and the operator norm. We denote the $\ellinf$-matrix norm of $A$ defined as
% \begin{align}
    $\nu_{A} = \vertiii{A}_{\infty} \triangleq \max_{j=1,\ldots,p}\sum_{j=1}^{p}\vert A_{ij}\vert.$ 
% \end{align}
We use $\Vert A\Vert_{1,\text{off}} = \sum_{i\neq j}\vert A_{ij}\vert$ to denote the off-diagonal $\ell_{1}$ norm. We use $\mvec(A)$ to denote the $p^2$-vector formed by stacking the columns of $A$ and use $\Gamma(A)=(I\otimes A)$ to denote the kronecker product of $A$ with the identity matrix $I$. For symmetric positive definite matrices $A_1$ and $A_2$, we use $A_1\succ A_2$ to denote  $A_1-A_2$ is positive definite.  We define $\text{sign}(A_{ij}) = +1$ if $A_{ij}>0$ and $\text{sign}(A_{ij}) = -1$ if $A_{ij}<0$. For two-real valued functions $f(\cdot)$ and $g(\cdot)$, we write $f(n) = \mathcal{O}(g(n))$ if $f(n)\leq cg(n)$ and $f(n) = \Omega(g(n))$ if $f(n)\geq c^{\prime}g(n)$ for constants $c,c^{\prime}>0.$

\section{Problem Setup}
\label{sec: problem-setup}
Consider a $p$-dimensional random vector $X$ following the Gaussian distribution $\mathcal{N}(0,\Sigma_X)$ with a known covariance matrix $\Sigma_X\succ 0$ (we outline a relaxation of this assumption in Remark~\ref{rmk: unknown cov matrix}). Let $Y = ({B^{*}})^{-1}X$ with a symmetric $p\times p$ matrix $B^*\succ 0$ and note that $Y\sim \mathcal{N}(0,{\Theta^{*}}^{-1})$, where $\Theta^*=B^*\Sigma_XB^*$. Define the sample covariance matrix $S=n^{-1}\sum_{i=1}^nY_iY_i^\transpose$, where $\{Y_1,\ldots,Y_n\}$ are the $n$ (possibly $n<p$) i.i.d. samples of $Y$. For some $\lambda_n>0$, we consider the $\ell_1$ regularized MLE for estimating $B^*$: 
\begin{align}\label{eq: log-det0}
    \argmin_{B\succ 0; \Theta=B\Sigma^{-1}_{X} B^\transpose  }\left[\Tr(S\Theta)-\log\det(\Theta)+\lambda_{n}\Vert B\Vert_{1,\text{off}}\right], 
\end{align}
where $\Vert B\Vert_{1,\text{off}}=\sum_{i\neq j}\vert B_{ij}\vert$ is the $\ell_1$-norm applied to the off-diagonal entries of $B\in \mathbb{R}^{p\times p}$. The loss function in \eqref{eq: log-det0} without the $\ell_1$ penalty is the negative log-likelihood of $Y$, and maximizing it to estimate $B^*$ yields an unrestricted MLE.

The optimization problem in \eqref{eq: log-det0} looks similar to the $\ell_1$-regularized log-determinant problem, which has a rich, long history in high-dimensional statistics, machine learning, signal processing, and network sciences (see for instance \cite{maathuis2018handbook, banerjee2008model, yuan2007model, friedman2008sparse}). The bulk of this literature focuses on estimating $\Theta^*$. The resultant estimator, referred to as the graphical LASSO (or GLASSO), has many nice theoretical properties (e.g., asymptotic consistency and support recovery in the high-dimensional regime) \cite{ravikumar2011high, rothman2008sparse, zhang2014sparse}. However, our estimator in \eqref{eq: log-det0} is significantly different from GLASSO because we are estimating $B^*$ rather than $\Theta^*$. Other studies close to our setup estimate a sparse Cholesky factor of $\Theta^*$ \cite{huang2006covariance, cordoba2020sparse, jelisavcic2018fast}. Recall that the Cholesky decomposition is given by $\Theta^*=LL^\transpose$ where $L\succ0$ is a lower triangular matrix. We differ from this line of work on multiple fronts: (i) we do not require $B^*$ to be a lower or upper triangular Cholesky factor; (ii) our method allows for  arbitrary correlations between the nodal injections resulting in an extra $\Sigma_X^{-1}$ between the factors; and (iii) to the best of our knowledge, ours is the first work to provide guarantees on the sample complexity for estimating $B^*$ in the high-dimensional regime.

\smallskip 
\begin{remark}(\emph{Unknown covariance matrix $\Sigma_X$}).\label{rmk: unknown cov matrix}
In problem \eqref{eq: log-det0}, we assume that $\Sigma_X$ is known. If this is not the case, we can slightly modify \eqref{eq: log-det0} to estimate $B^*D$ instead of $B^*$, where $D$ is the unique square root of $\Sigma^{-1}_X$ satisfying $D^2=\Sigma^{-1}_X$. This approach works best if the sparsity of $B^*$ (approximately) equals the sparsity of $B^*D$, which for instance happens when $\Sigma_X$ is (approximately) diagonal.    \qed
\end{remark}

\begin{remark}(\emph{On invertibility of $B^*$}).\label{rmk: invertibility}
The invertiblity assumption of $B^*$ ensures that $B^*$ is identifiable from samples. This holds in several applications including the ones in (2)-(4) in Section~\ref{sec:intro}. However, this might not be true if $B^*$ is a Laplacian matrix that has $k$ zero eigenvalues. One common work around (see e.g., \citep{grone1990laplacian, DekaTSG2020, dorfler2012kron}) is to work with a reduced Laplacian matrix by deleting $k$ rows and columns of $B^*$ (we employ this insight in our experiments; see~Section \ref{sec: experiments}).  \qed
\end{remark}

\section{A Convex Estimator and Statistical Guarantees}\label{sec: Statistical guarantees}
In this section, we first recast the objective in \eqref{eq: log-det0} in terms of  $B$ for a known $\Sigma_X$. We then present our main results on the performance of our estimator in \eqref{eq: log-det0} when $X$ is Gaussian and non-Gaussian. We comment on extending { our results to other convex loss functions} and conclude with an overview of the key steps in proving our results. Full details are given in the Appendix. 

We begin by rewriting the problem in \eqref{eq: log-det0} in a form that is more suitable to our methods of analysis. Let $D$ be the unique square root of $\Sigma^{-1}_X$ satisfying $D^2=\Sigma^{-1}_X$ (see \cite{bhatia2009positive}). Substituting $B=B^\transpose$ and  $\Theta=BD^2B^\transpose$ in the cost function of \eqref{eq: log-det0} yields the following: 
\begin{align}\label{eq: log-det}
    \widehat{B} = \argmin_{B\succ0 }\left[\Tr(DBSBD)-\log\det(B^2)+\lambda_{n}\Vert B\Vert_{1,\text{off}}\right], 
\end{align}
where we use the fact that the trace operator is cyclic and the determinant of a matrix product equals the product of matrix determinants. We dropped constants that have no effect on the estimate. The symmetry and invertibility of $B$ is sufficient enough to ensure that $\log(\cdot)$ is well-defined. In other words, the positive-definiteness assumption is not needed for the well-posedness of \eqref{eq: log-det}. 

Lemma \ref{lma: uniq soln} below is the starting point of our analysis. It establishes two key properties of the estimator in \eqref{eq: log-det} under the positive definiteness of $B$: (i) loss function in \eqref{eq: log-det} is convex in $B$ and (ii) $\widehat{B}$ is unique.
\begin{lemma}{{\emph(convexity and uniqueness)}}\label{lma: uniq soln}
For any $\lambda_{n}\!>\!0$ and $B\!\succ\! 0$, (i) the $\ell_{1}$-log determinant problem in \eqref{eq: log-det} is convex and (ii) $\widehat{B}$ in \eqref{eq: log-det} is the unique minima satisfying the sub-gradient condition $2D^{2}\widehat{B}S - 2\widehat{B}^{-1}\!+\!\lambda_{n}\widehat{Z}\!=\!0$. Here $\widehat{Z}$ belong to the sub-gradient $\partial\Vert \widehat{B} \Vert_{1,\text{off}}$ so that $\widehat{Z}_{ij}=0$, for $i=j$, $\widehat{Z}_{ij}\!=\!\mathrm{sign}(\widehat{B}_{ij})$ when $\widehat{B}_{ij}\neq 0$ and $|\widehat{Z}_{ij}|\leq 1$ when $\widehat{B}_{ij}=0$, for $i \ne j$.  

\end{lemma}

A few comments of Lemma \ref{lma: uniq soln} are in order. ({ \emph{Convexity}}) First, we recall that the compositions of two convex functions is in general not convex. As an example, consider two convex functions $f(x)=x^2$ and $g(x)=-x$, however, $g(f(x))=-x^2$ is not convex. Therefore, in light of the fact that the loss function is a composite function of $B$, it is not clear if \eqref{eq: log-det} is convex. Nonetheless, in Lemma \ref{lma: uniq soln} we prove that \eqref{eq: log-det} is a convex program. Key to our proof is the notion of monotone convex functions. (\emph{Uniqueness}) Second, the uniqueness result is non-trivial in high-dimensions $(n<p)$ because the Hessian is rank deficient, and hence, the loss function in \eqref{eq: log-det} might not be be strictly convex. However, in Lemma \ref{lma: uniq soln} (ii) we show that $\widehat{B}$ is unique. Key to our proof is the 
notion of coercivity and it adapts techniques in \cite{ravikumar2011high} to the case where the objective function is quadratic in the optimization variable. 

\subsection{Statement of Main result}\label{sec: Main results}
Our first result theoretically characterizes the performance of $\widehat{B}$ in \eqref{eq: log-det} when $Y$ is Gaussian. Furthermore, our second result provides such a characterization for $\widehat{B}$ when $Y$ is non-Gaussian. At a crude level, our results guarantee that when the number of samples $n$ scales as $d^2\log p$, our $\ell_1$-regularized MLE correctly recovers the support of $B^*$ and is close to $B^*$ (measured in Frobenius and operator norms) with high probability. Here, $d$ is the degree of the graph underlying $B^*$.

Since we consider an $\ell_1$ regularized log-determinant program for our ML estimator \eqref{eq: log-det}, our results might appear similar to that of 
\citep{ravikumar2011high}. However, as also pointed in Section \ref{sec: outline}, our main results, including the assumptions and sufficient conditions needed to derive them, are not subsumed by those in \citep{ravikumar2011high}, or vice versa (see Remark \ref{rmk: GLASSO comparison} and Section \ref{sec: experiments} for more thorough discussion).
 
We begin with the assumptions that are essential to prove our theoretical statements. Similar subset of assumptions in the context of $\ell_1$ regularized least squares problem appeared in \citep{wainwright2009sharp, tropp2006just,meinshausen2006high,zhao2006model}, and in the context of $\ell_1$ regularized inverse covariance estimation problem appeared in \cite{ravikumar2011high, zhang2014sparse}. We define the edge set $\mathcal{E}(B^{*}) = \{(i,j): B^{*}_{ij}\neq 0, \text{for all} \hspace{3px} i\neq j\}$. Let $E:=\{\mathcal{E}(B^{*})\cup (1,1)\ldots\cup (p,p)\}$ be the augmented set including the diagonal elements. Let $E^{c}$ be the complement of $E$.\\

\textbf{[A1] Mutual incoherence condition.}\label{mutual incoherence}  Let $\Gamma^{*}$ be the Hessian of the log-determinant function in \eqref{eq: log-det}: 
\begin{align}\label{eq: hessian log-det}
    \Gamma^{*} \triangleq \nabla^{2}_{B}\log\det(B)\vert_{B=B^{*}} = {{B^{*}}^{-1}}\otimes {{B^{*}}^{-1}}. 
\end{align}
For $\Gamma^{*}$ in \eqref{eq: hessian log-det}, there exists some $\alpha \in (0,1]$ such that 
%\begin{align}
    $\vertiii{\Gamma^{*}_{E^{c}E}(\Gamma^{*}_{EE})^{-1}}_{\infty}\leq 1-\alpha$. 
%\end{align}

\textbf{[A2] Hessian regularity condition.} Let $d$ be the maximum number of non zero entries among all the rows in $B^*$ (i.e., the degree of the underlying graph), $\Theta^*=B^*\Sigma^{-1}_{X}B^*$, and $D^2=\Sigma^{-1}_{X}$. Then, 
\begin{align}\label{eq: Hessian regularity}
    \vertiii{{\Gamma^{*}}^{-1}}_{\infty}\leq \frac{1}{4d\Vert {\Theta^{*}}^{-1}\Vert_{\infty}\vertiii{D^{2}}_{\infty}}.
\end{align}

\textbf{[A3] Maximum row norm condition.}\label{eq: max row norm} There exists a constant $c>0$ for which 
$\vertiii{B^{*}}_{\infty}\geq c$, or equivalently, the spectral norm is bounded as $ \vertiii{B^{*}}_{2} \geq c/\sqrt{p}$. \\

A few comments are in order. [\textbf{A1}] Our novel incoherence condition on $B^*$ regulates the influence of irrelevant variables (elements of Hessian restricted to $E^c\times E$) on relevant variables (elements of Hessian restricted to $E\times E$). The $\alpha$-incoherence assumption of the above type is standard in literature, and \citep{ravikumar2011high} demonstrates its validity for several graphs, including chain and grid graphs, which we will explore in experimental section. Notice that the $\alpha$-incoherence in \citep{ravikumar2011high} is imposed on $\Theta^*$. Instead, we require it on $B^*$. [\textbf{A2}] This condition is in parallel with bounding the maximum eigenvalue of ${\Theta^{*}}^{-1}$ condition for estimating sparse $\Theta^*$ (see for e.g., \citep{rothman2008sparse,lam2009sparsistency}).
[\textbf{A3}] This is a type of signal-to-noise ratio condition and is unique to our problem. It says that the energy (measured in 2-norm ) in the signal $B^*$ should be greater than a certain threshold. 

Our problem set-up assumes that the node potentials $Y_{i}$, $i\in [p]$, are Gaussian. However, we work with sub-Gaussian distributions, a natural generalization to the Gaussian case, which encompasses many well known distributions that occur in practice (for e.g., bounded random variables, gaussians and mixture of gaussians). We define this distributional assumption below.    
\begin{definition}(Sub-Gaussian random variable) A zero mean random variable $Z$ is said to be sub-Gaussian if there exists a constant $\sigma>0$ such that for any $t\in \mathbb{R}$, 
\(\mathbb{E}[\exp(tZ)]\leq \exp\left({{\sigma^{2}t^{2}}/{2}}\right)\). 
\end{definition}

Our first main result below provides sufficient conditions on the number of samples $n$ needed for $\widehat{B}$ in \eqref{eq: log-det} to exactly recover the sparsity structure of $B^*$ and to achieve sign consistency, defined as $\text{sign}(\widehat{B}_{ij}) = \text{sign}(B^{*}_{ij})$, for all $(i,j)\in E$. We recall that  $\nu_{A} = \vertiii{A}_{\infty} \triangleq \max_{j=1,\ldots,p}\sum_{j=1}^{p}\vert A_{ij}\vert$ and define $\Sigma^{*} = {\Theta^{*}}^{-1}$ to be the covariance matrix of the node potential $Y$. 

\begin{theorem}{(Support Recovery: Sub-Gaussian)}\label{thm: sub-Gaussian support recovery} Let $Y=(Y_1,\ldots,Y_p)$ be the node potential vector. 
Suppose that $Y_{i}/\sqrt{\Sigma^{*}_{ii}}$ is sub-Gaussian with parameter $\sigma$  and assumptions [\textbf{A1-A3}] hold. Let the regularization parameter $\lambda_{n} = C_{0}\sqrt{\tau(\log 4p)/n}$, where $C_0$ is given below. If the sample size 
$n>C^2_{1}d^{2}(\tau\log p+\log 4)$, the following hold
with probability at least $1-\frac{1}{p^{\tau-2}}$, for some $\tau>2$:
\begin{enumerate}[label=(\alph*)]
\item  $\widehat{B}$ exactly recovers the sparsity structure of $B^{*}$; that is, $\widehat{B}_{E^{c}} = 0$,
\item $\widehat{B}$ satisfies the element-wise $\ellinf$ bound $\Vert \widehat{B} - B^{*}\Vert_{\infty}\leq C_{2}\sqrt{\frac{\tau\log p+\log 4}{n}}$, and
\item $\widehat{B}$ satisfies sign consistency if $\vert B^{*}_{\min} \vert\geq 2C_{2}\sqrt{\frac{\tau\log p + 4}{n}}$, where $B^{*}_{\min}\triangleq \min_{(i,j)\in \mathcal{E}(B^{*})}\vert B^{*}_{ij}\vert$, 
\end{enumerate}
where $C_{1}=192\sqrt{2}\left[(1+4\sigma^{2})\max\limits_{i}(\Sigma^{*}_{ii})\nu_{D^{2}}\nu_{B^{*}}\right]\max\{\nu_{{\Gamma^{*}}^{-1}}\nu_{{B^{*}}^{-1}},2\nu^{2}_{{\Gamma^{*}}^{-1}}\nu^{3}_{{B^{*}}^{-1}},2\alpha^{-1}d^{-1}\}$, $C_{2} = [64\sqrt{2}(1+4\sigma^{2})\max\limits_{i}(\Sigma^{*}_{ii})\nu_{{\Gamma^{*}}^{-1}}\nu_{D^{2}}\nu_{B^{*}}]$, and $C_{0} = C_{2}/(4\nu_{{\Gamma^{*}}^{-1}})$.
\end{theorem}

The quantities $(\nu_{{\Gamma^{*}}^{-1}},\nu_{D^{2}},\nu_{B^{*}},\nu_{{B^{*}}^{-1}})$ capture the inherent complexity of the model and do not depend on the number of samples $n$. As long as the magnitude of the entries in ${\Gamma^{*}}^{-1}, D^{2}$, and $B^{*}$ scale as $\mathcal{O}(1/d)$, the model complexity parameters do not depend on $(p,d)$. {That is, as the size of the network grows with $(p,d)$ the edge strengths decay with $d$}. Suppose that the model complexity parameters are constants and that $n=\Omega(d^{2}\log p)$. Then part (a) of Theorem \ref{thm: sub-Gaussian support recovery} guarantees that our ML estimator does not falsely include entries (or edges in the underlying graph) that are not in the support of $B^*$. Part (b) establishes the element-wise $\ellinf$ norm consistency of $\widehat{B}$; that is, $\Vert \widehat{B}-B^{*}\Vert_{\infty} = \mathcal{O}(\sqrt{(\log p)/n})$. Finally, part (c) establishes sign consistency of $\widehat{B}$, and hence, our estimator does not falsely exclude entries that are in the support of $B^*$. Crucial is the requirement of $\vert B^{*}_{\min}\vert = \Omega\left(\sqrt{(\log p)/n}\right)$, which puts a limit on the minimum (in absolute) value of the entries in $B^*$. This condition parallels the familiar \emph{beta-min} condition in the LASSO literature (see \citep{wainwright2009sharp, van2008high}). 
 
We now present a corollary to Theorem \ref{thm: sub-Gaussian support recovery} that gives consistency rates of convergence for $\widehat{B}$ in the Frobenius and operator norms. Let $\mathcal{E}(B^{*}) = \{(i,j): B^{*}_{ij}\neq 0, \text{for all} \hspace{3px} i\neq j\}$ be the edge set of $B^{*}$. 
\begin{corollary}\label{cor: corollary1}
Let $s = \vert\mathcal{E}(B^{*})\vert$ be the cardinality of $\mathcal{E}(B^{*})$. Under the same hypotheses in Theorem \ref{thm: sub-Gaussian support recovery}, with probability greater than $1-\frac{1}{p^{\tau-2}}$, the estimator $\widehat{B}$ satisfies 
\begin{align*}
    \Vert \widehat{B} - B^{*}\Vert_{F} &\leq C_{2}\sqrt{\frac{(s+p)(\tau\log p + 4)}{n}} \,\, \text{ and }\,\,
    \Vert{\widehat{B} - B^{*}}\Vert_{2} \leq C_{2}\min\{d,\sqrt{s+p}\}\sqrt{\frac{\tau\log p + 4}{n}}. 
\end{align*}
\end{corollary}

\emph{Proof sketch}. Both the Frobenius and operator norm bounds follows by applying standard matrix norm inequalities to the $\ellinf$ consistency bound in part (b) of Theorem \ref{thm: sub-Gaussian support recovery}. Importantly, $s+p$ is the bound on the maximum number of non-zero entries in $B^*$, where $s$, by definition, is the total number of off-diagonal non-zeros in $B^*$. Complete details are provided in the Appendix.\\

Thus far we have assumed that the nodal potentials $Y_{i}$ are sub-Gaussian random variables. We now explore another broad class of random variables with bounded $k^{\text{th}}$ moments, which are known to have tails that decay according to some power law \citep{petrov1995limit}. An important example of power law distributions are Pareto distributions which finds applications in a wide variety of areas \citep{newman2005power,milojevic2010power}. Motivated by such important practical considerations, we state our next result for random variables with bounded moments. We begin with the following definition.
\begin{definition}(Bounded moments)\label{def: Bounded moments} A random variable $Z$ is said to have bounded $4k^{\text{th}}$ moment if there exists a constant $M_{k}\in \mathbb{R}$ such that 
%\begin{align}
    $\mathbb{E}\left[(Z)^{4k}\right]\leq M_{k}$. 
%\end{align}
\end{definition}

Results below parallel Theorem \ref{thm: sub-Gaussian support recovery} and Corollary \ref{cor: corollary1} for random variables with bounded moments. 

\begin{theorem}(Support Recovery: Bounded Moments)\label{thm: bounded moment support recovery} Let $Y=(Y_1,\ldots,Y_p)$ be the node potential vector. Suppose that $Y_{i}/\sqrt{\Sigma^{*}_{ii}}$ has bounded moment as in Definition \ref{def: Bounded moments} and assumptions [\textbf{A1-A3}] hold. Let the regularization parameter $\lambda_{n}= C_{0}\sqrt{\tau(\log 4p)/n}$, with $C_{0}$ defined in Theorem \ref{thm: sub-Gaussian support recovery}. If the sample size  $n > C_{4}d^{2}p^{\tau/k}$. Then with probability greater than $1-{1}/{p^{\tau-2}}$, for some $\tau>2$, the following hold: (a) $\widehat{B}$ exactly recovers the sparsity structure of $B^{*}$ (that is $\widehat{B}_{E^{c}} = 0$); (b) the element-wise $\ellinf$ bound $\Vert \widehat{B} - B^{*}\Vert_{\infty}\leq C_{5}\sqrt{\frac{p^{\tau/k}}{n}}$; and (c) $\widehat{B}$ satisfies sign consistency if $\vert B^{*}_{\min}\vert\geq 2C_{5}\sqrt{\frac{p^{\tau/k}}{n}}$. 
 
\end{theorem}
The constants and their dependence on the model complexity parameters are given in the Appendix.
 
\begin{corollary}\label{cor: corollary2}
Suppose the hypotheses in Theorem \ref{thm: bounded moment support recovery} hold.
Then with probability greater than $1-\frac{1}{p^{\tau-2}}$:  $\Vert \widehat{B} - B^{*}\Vert_{F} \leq C_{5}\sqrt{\frac{(s+p)(p^{\tau/k})}{n}}$ and $
    \Vert{\widehat{B} - B^{*}}\Vert_{2}\leq C_{5}\min\{d,\sqrt{s+p}\}\sqrt{\frac{p^{\tau/k}}{n}}$, where $s = \vert\mathcal{E}(B^{*})\vert$.

\end{corollary}

Interpretations of Theorem \ref{thm: sub-Gaussian support recovery} and Corollary \ref{cor: corollary1} also hold for Theorem \ref{thm: bounded moment support recovery} and Corollary \ref{cor: corollary2}. However, in this setting, we have different sample size $n=\Omega(d^{2}p^{\tau/k})$ and $\vert B^{*}_{\min}\vert =\Omega(\sqrt{p^{\tau/k}/n})$, where $k$ is given by Definition \ref{def: Bounded moments}. In contrast, for sub-Gaussian case we have logarithmic dependence in $p$ (the number of vertices). Finally, albeit fundamentally different from GLASSO estimator, we were able to obtain consistency rates for  
$\widehat{B}$ \eqref{eq: log-det} that are similar to those in \citep{ravikumar2011high, cai2011constrained}. 
\smallskip 
\begin{remark}\label{rmk: GLASSO comparison}\emph{(Comparison with the GLASSO estimator).} For simplicity, suppose that $\Sigma_{X}$ is diagonal. Then, it follows that $B^{*}\succ0$ is the unique square root of $\Theta^{*} = (B^{*})^{2}$. Thus, a na\"{i}ve way to estimate $B^*$ is by taking the square root of the GLASSO estimate $\widehat{\Theta}$. Let us call this estimator $\widehat{B}_{SR}$ and note that $\widehat{B}_{SR}$ inherits its optimal properties from $\widehat{\Theta}$. We show that $\widehat{\Theta}$ has sub-optimal estimation rate than $\widehat{B}$ in \eqref{eq: log-det} for estimating $B^*$. Let $B^{*}$ contains $d$ non-zero elements in every row. Then the underlying graph of $\Theta^{*}$ is a two-hop network with degree $d^{4}$. Using sample complexity results from \citep{ravikumar2011high}, it follows that $\widehat{\Theta}$ requires $n=\Omega(d^{4}\log p)$ to estimate $B^{*}$. Instead, our $\ell_1$-regularized MLE requires $n=\Omega(d^{2}\log p)$ samples. This reduction is more pronounced for networks with a large degree $d$. \qed 
\end{remark}

\subsection{Outline of Main Analysis}\label{sec: outline}
We provide an outline of our methods and main strategies to prove Theorem \ref{thm: sub-Gaussian support recovery}. We employ the \emph{primal-dual witness technique}---a well-known method used to derive statistical guarantees for sparse convex estimators \citep{wainwright2009sharp, wainwright2019high,loh2017support}.  This technique involves constructing a primal-dual pair $(\widetilde{B},\widetilde{Z})$ satisfying the zero-subgradient condition of the convex problem in \eqref{eq: log-det}, such that (the primal) $\widetilde{B}$ has the correct (signed) support. Suppose this construction succeeds, from the uniqueness result in Lemma \ref{lma: uniq soln}, it follows that $\widehat{B}=\widetilde{B}$, and the dual $\widetilde{Z}$ is an optimal solution to the dual of \eqref{lma: uniq soln}. Thus, at the heart of our analysis is in showing that the primal-dual construction succeeds with high-probability. Similar technique is also used to prove Theorem \ref{thm: bounded moment support recovery} (i.e., the non-Gaussian case); see Appendix. 

While our proof methods are inspired from \citep{ravikumar2011high, wainwright2009sharp}, our analysis is more involved due to the presence of $B^2$, as opposed to $B$, in the loss function of \eqref{eq: log-det}. Consequently, we require more nuanced assumptions (as in \textbf{[A2]-[A3]}) and dual feasiblity condition than the ones in \citep{ravikumar2011high} (see below).
 
\subsection{Primal-dual pair and supporting lemmas}\label{sec: PDW}
We briefly introduce the primal-dual witness construction. In Lemma \ref{eq: suff cond}, we provide sufficient conditions under which this construction succeeds.

We construct the primal-dual pair $(\widetilde{B},\widetilde{Z})$ as follows. The primal solution $\widetilde{B}$ is determined by solving
\begin{align}\label{eq: rest log-det}
    \widetilde{B} \triangleq \argmin_{B=B^{T},B\succ 0,B_{E^{c}}=0}\left[\Tr(DBSBD) - \log\det(B^{2})+\lambda_{n}\Vert B\Vert_{1,\text{off}}\right]. 
\end{align}
Here \eqref{eq: rest log-det} is a restricted problem in that we impose ${B}_{E^{c}} = 0$. Also, we have $\widetilde{B}\succ 0$ and $\widetilde{B}_{E^{c}}=0$. The dual $\widetilde{Z}\in \partial \Vert \widetilde{B}\Vert_{1,\text{off}}$ is chosen such that it satisfies the zero-subgradient condition of \eqref{eq: rest log-det}. This is obtained by setting $2\lambda_{n}\widetilde{Z}_{ij} = [\widetilde{B}^{-1}]_{ij} - [D^{2}\widetilde{B}S]_{ij}$, 
for all $(i,j)\in E^{c}$. It can be verified that $(\widetilde{B},\widetilde{Z})$ satisfies the zero-subgradient condition (see the statement of Lemma \ref{lma: uniq soln}) of the original problem in \eqref{eq: log-det}. Thus, it remains to establish the strict dual feasibility condition; that is $\vert \widetilde{Z}_{ij}\vert<1$, for any $(i,j)\in E^{c}$. 

We introduce some notation. Let $W\triangleq S-{\Theta^{*}}^{-1}$, where $S$ is the sample covariance and ${\Theta^{*}}^{-1}$ is the true covariance of $Y$. Let $\Delta \triangleq \widetilde{B}-B^{*}$ be a measure of distortion between the primal solution $\widetilde{B}$ as defined in equation \eqref{eq: rest log-det} and the true matrix to be estimated $B^{*}$.  We also need the 
the higher order terms of the Taylor expansion of the gradient $\nabla \log\det(\widetilde{B})$ centered around $B^{*}$ \citep{Boyd}:  
\begin{align}
  \nabla \log\det(\widetilde{B})&=  {B^{*}}^{-1} + {B^{*}}^{-1}\Delta {B^{*}}^{-1}+\underbrace{\widetilde{B}^{-1}-{B^{*}}^{-1} - {B^{*}}^{-1}\Delta {B^{*}}^{-1}}_{\triangleq R(\Delta)}. \label{eq: reminder term}
\end{align}

\begin{lemma}\label{eq: suff cond}(Sufficient conditions for strict dual feasibility) Let the regularization parameter $\lambda_n>0$ and $\alpha$ be defined as in {\bf [A1]}. Suppose the following holds 
\begin{align}\label{eq: strict duality conditions}
    \max\left\lbrace\vertiii{\Gamma(D^{2}\Delta)+\Gamma(D^{2}B^{*})}_{\infty}\Vert W\Vert_{\infty},\Vert R(\Delta)\Vert_{\infty},\vertiii{\Gamma(D^{2}\Delta)}_{\infty}\Vert {\Theta^{*}}^{-1}\Vert_{\infty}\right\rbrace\leq \frac{\lambda_{n}\alpha}{24}. 
\end{align}
Then the dual vector $\widetilde{Z}_{E^{c}}$ satisfies $\Vert \widetilde{Z}_{E^{c}}\Vert_{\infty}<1$, and hence, $\widetilde{B} = \widehat{B}$. 
\end{lemma}
\emph{Proof sketch}: The proof essentially involves expressing the sub-gradient condition in Lemma \ref{lma: uniq soln} as a vectorized form using $R(\Delta)$ (in \eqref{eq: reminder term}) and $W$. By manipulating the vectorized sub-gradient condition, we obtain an expression of $\widetilde{Z}_{E^{c}}$
that is a function of the quantities in \eqref{eq: strict duality conditions}. We finish off the proof by repeated applications of triangle inequality of norms and invoking assumptions in Lemma \eqref{eq: strict duality conditions}. 

The following results provides us with dimension and model complexity dependent bounds on the reminder term $R(\Delta)$ in \eqref{eq: reminder term} and the distortion $\Delta$.  
 
\begin{lemma}({Control of reminder})\label{lma: control remainder}  Suppose that the element-wise $\ellinf$-bound $\Vert \Delta\Vert_{\infty}\leq \frac{1}{3\nu_{{B^{*}}^{-1}}d}$ holds, then the matrix $Q = \sum\limits_{k=0}^{\infty}(-1)^{k}({B^{*}}^{-1}\Delta)^{k}$ satisfies the bound $\nu_{Q^{T}}\leq \frac{3}{2}$ and the matrix $R(\Delta) = {B^{*}}^{-1}\Delta{B^{*}}^{-1}\Delta Q {B^{*}}^{-1}$ has the element-wise $\ellinf$-norm bounded as 
 \begin{align}
     \Vert R(\Delta)\Vert_{\infty}\leq \frac{3}{2}d\Vert \Delta\Vert^{2}_{\infty}\nu_{{B^{*}}^{-1}}^{3}.
\end{align}
\end{lemma}
The proof, adapted from \citep{ravikumar2011high}, is algebraic in nature and relies on certain matrix expansions. The details are provided in Appendix. In the following result, we provide a sufficient condition under which the element-wise $\ellinf$-bound on $\Delta$ in Lemma \ref{lma: control remainder} holds.

% The proof is adapted from \citep{ravikumar2011high}, where a similar result is derived using matrix expansion techniques. In the next lemma, we show that the distortion measure $\Delta$ can be controlled. More specifically, we say that $\Delta$ lies inside an $\ellinf$ ball with a specific radius $r$ mentioned in the lemma. 
\begin{lemma}({Control of $\Delta$})\label{lma: control of delta} Let $r\!\triangleq\!4\nu_{{\Gamma^{*}}^{-1}}\left[\nu_{D^{2}}\nu_{B^{*}}\Vert W\Vert_{\infty}\!+\!0.5{\lambda_{n}}\right]\leq \min\Big\{\frac{1}{3\nu_{{B^{*}}^{-1}}d},\frac{1}{6\nu_{{\Gamma^{*}}^{-1}}\nu_{{B^{*}}^{-1}}^{3}d}\Big\}$. Then we have the element-wise $\ellinf$ bound $\Vert \Delta\Vert_{\infty} = \Vert \widetilde{B}-B^{*}\Vert_{\infty}\leq r$.
\end{lemma}
\emph{Proof sketch}: By construction $\widetilde{B}_{E^c}=\widetilde{B}^*_{E^c}=0$. Hence, $\|\Delta\|_\infty=\|\Delta_E\|_\infty$, where $\Delta_E=\widetilde{B}_E-{B}_E^*$. We construct a continous vector valued function $F:\Delta_E\to \Delta_E$ that has a unique fixed point. Invoking assumptions {\bf [A1]-[A2]}, we show that $F(\cdot)$ is a contractive map on the $\ellinf$ ball defined as $\mathbb{B}_{r} = \{A: \Vert A\Vert_{\infty}\leq r\}$ with $r$ defined in the statement of the lemma. Specifically, we show that $F(\mathbb{B}_{r})\subseteq \mathbb{B}_{r}$. Finally, we finish off the proof by an application of Brower's fixed point theorem \citep{kellogg} to show that the unique fixed point is inside $\mathbb{B}_{r} $. Consequently, $\Vert \Delta\Vert_{\infty}\leq r$. 

Finally, the result in Theorem \ref{thm: sub-Gaussian support recovery} follows by putting these lemmas together for an appropriate choice of $\lambda_n$ and the sample size requirement, and there upon, invoking some known concentration inequalities. 

\section{Experiments}\label{sec: experiments}
We validate the support recovery performance of our $\ell_1$-regularized MLE on synthetic and a benchmark power distribution network (see Fig.~\ref{fig: network_illustration}).
We choose $\lambda_{n}$ proportional to $\sqrt{{\log p}/{n}}$. Our results are averaged over $100$ trials of $n$ independent samples of $Y$.
We compare our $\ell_1$-regularized MLE performance with (i) the square-root estimator (hereafter, GLASSO+SR) that identifies the support of $B^*$ by determining $(i,j)$ for which ${\widehat{\Theta}}^{\frac{1}{2}}_{i,j}\ne 0$; and (ii) the GLASSO+2HR (Hop Refinement) estimator \citep{DekaTSG2020} that identifies the support of $B^*$ by determining $(i,j)$ for which $\widehat{\Theta}_{i,j}\leq -\tau $ for $\tau=1e-02$. Here $\widehat{\Theta}$ is the GLASSO estimate of the inverse covariance matrix of $Y$ \citep{friedman2008sparse}. These estimators are described in detail in Introduction. To have a fair comparison with the GLASSO based estimators, we set $\Sigma_X$ ($X$ is the injected vector) to be diagonal. However, as discussed earlier, our $\ell_1$ regularized ML estimator works for any $\Sigma_X\succ 0$.

We consider $p$ to be as large as $64$ nodes,  Computational examples involving large data matrices for $B^*$ having a lower triangular matrix form has been reported in \citep{jelisavcic2018fast}. 

\emph{(i) Synthetic data}: We consider two undirected graphs for $B^*$, the chain graph and the grid graph for $p=\{32,64\}$ nodes. We set $B^{*}_{ij} = 1$ for $(i,j)\in E$ and $B^{*}_{ij} = 0$ for $(i,j)\in E^{c}$, where $E$ can be the edge set of the chain or grid graph. We then adjust the diagonal elements of $B^{*}$ to ensure $B^*\succ 0$.   

\emph{(ii) Power network:} We set $B^*$ to be the Laplacian of the IEEE 33 bus power distribution network \citep{zimmerman2010matpower}. For this data, we note that $B^*$ is non-invertible because one of one zero eigenvalue. We obtain the reduced $B^*$ by deleting the first row and column of $B^*$. We also slightly modify the network by adding three loops of cycle length three, two of cycle length four, and one loop of cycle length five (see Fig \ref{fig: network_illustration}). We made these modifications to highlight that $\ell_1$-MLE imposes no connectivity assumptions on the graph underlying $B^*$, except sparsity. In contrast, GLASSO+2HR estimator \citep{DekaTSG2020} restricts the graph underlying $B^*$ from having cycles of length three (i.e., triangle-free). 

In Fig.~\ref{fig: support_recovery}, we show empirical support recovery probabilities for all three estimators as a function of the number of samples $n$. Both on synthetic and power network data, our $\ell_1$-MLE achieved superior rates than the other competing estimators. In fact, $\ell_1$-MLE exactly recovers the support of $B^*$ when the number of samples is in the order of $d^2\log p$, which is in excellent agreement with the proposed theory. Instead, for a similar performance, GLASSO+SR needed $d^4\log p$ samples (see Remark \ref{rmk: GLASSO comparison}).

We implement all three estimators using CVXPY 1.2 open source python package on Google Colab. All the simulation results reported in this paper can be reproduced using the code available at
\url{https://github.com/AnirudhRayas/SLNSCL}.

\setlength{\textfloatsep}{10pt plus 1.0pt minus 2.0pt}
\begin{figure}
    \centering
    \includegraphics[width=0.95\linewidth]{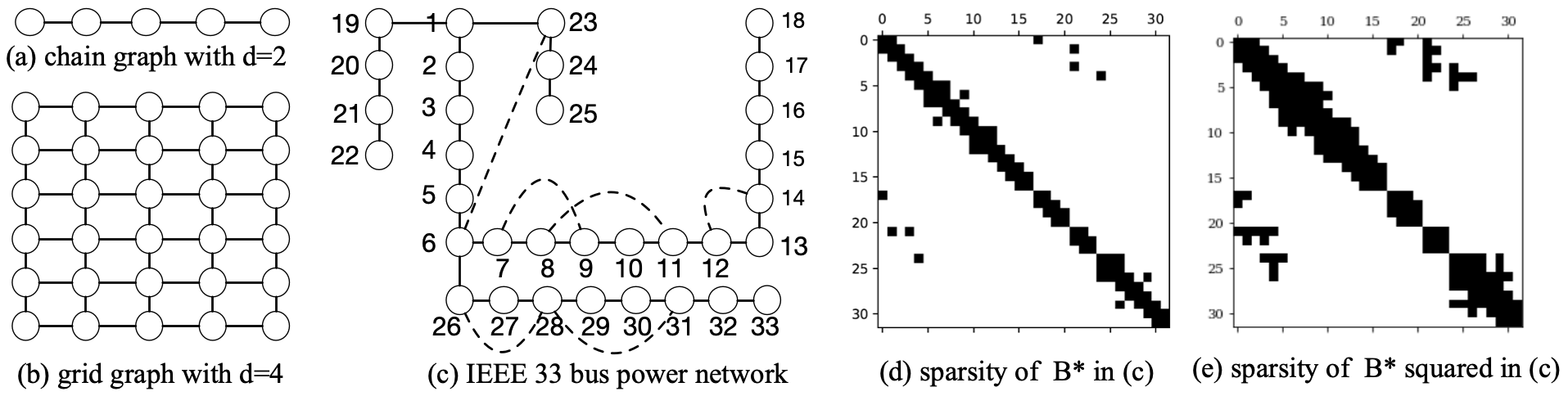}
    \caption{\footnotesize{Graphs used in experiments. (a) Chain graph with maximum degree $d=2$. (b) Grid graph with $d=4$. (c) IEEE 33 bus (node) distribution network with additional loops (shown in dashed lines). (d) Sparsity of $B^*$ associated with the IEEE 33 bus network. (d) Sparsity of ${(B^*)}^2$.} Notice that ${(B^*)}^2$ is denser relative to $B^*$. Consequently, GLASSO+HR needs more samples than $\ell_1$-MLE to recover the support (see plot below).}
    \label{fig: network_illustration}
\end{figure}
\begin{figure}
    \centering
    \includegraphics[width=0.95\linewidth]{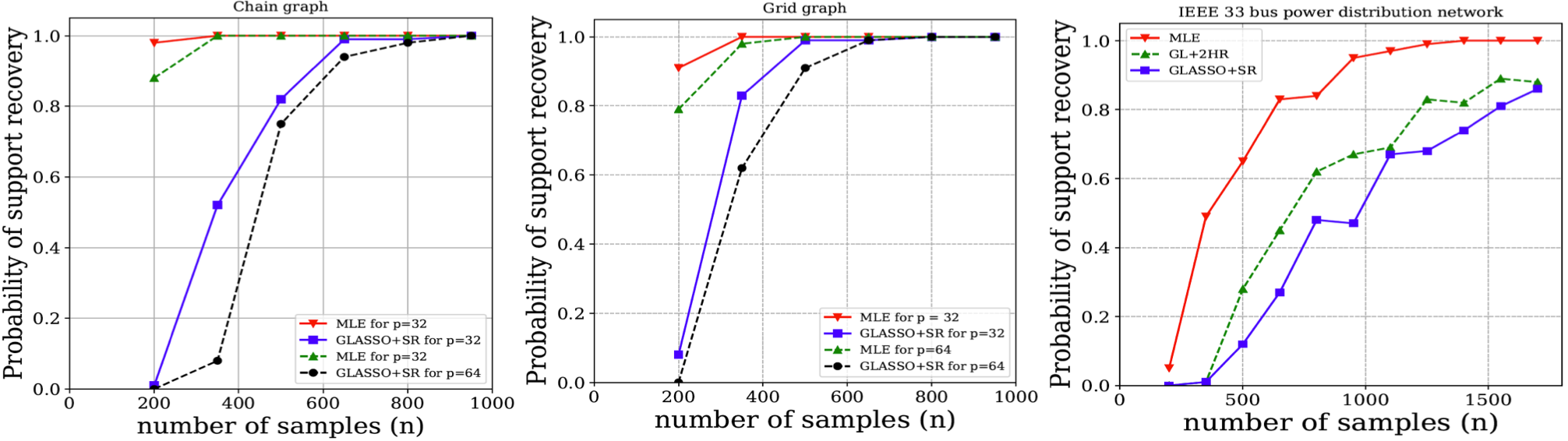}
    \caption{\footnotesize{Empirical probability of success of various estimators versus the raw sample size $n$ for (left) chain graph, (middle) grid graph, and (right) IEEE 33 bus network. For chain and grid graph, we compare our $\ell_1$ regularized MLE performance with GLASSO+SR for $p\in \{32, 64\}$. Instead, for IEEE 33 bus network, we compare $\ell_1$ regularized MLE with GLASSO+SR and GLASSO+2HR.}}
    \label{fig: support_recovery}
\end{figure}

\section{Discussions and Future Work}\label{sec: discussion}

Large networked systems obeying conservation laws of the form $X=B^*Y$ are often used to model and study interactions among different conserved quantities in various engineering and scientific disciplines. For such systems, we design a novel estimator of the unknown structure of the network (i.e., sparsity pattern of $B^\ast$) using an $\ell_1$-regularized maximum likelihood estimator. Our estimator only needs samples of the node potentials $Y$ and some knowledge of the statistics of the node injections $X$. We showed that this estimator is well defined under certain natural conditions by showing that the corresponding convex optimization has a unique optimum. We  established sparsistency and norm consistency of our estimator under a novel mutual incoherence condition. We then provided several numerical results that not only validated our theory but also showed that our proposed estimator outperforms several state of the art techniques for structure recovery in such systems. 

In our framework, we neither require knowledge of the actual injected flows ($X$), nor do we need $B^*$ to be a Laplacian matrix; this  allows our framework to be general enough to be applicable for a variety of domains ranging from electrical networks to social networks. Consequently, our framework and theoretical results admit many  extensions and refinements, such as recasting \eqref{eq: log-det0} as the minimization of the Bregman divergence for more general loss functions. In this work we restricted $B^*$ in the model $B^*Y-X=0$ to be invertible and positive definite. In several applications such as transportation, hydrodynamic, and neuronal networks $B^*$ might not be symmetric or non-normal \cite{asllani2018structure}, and hence, not positive definite. Extending our analysis to these cases could be a fruitful avenue for future work. Another area for future exploration is to consider practically relevant and methodologically challenging problems in various systems that demands network reconstruction from incomplete \cite{anguluri2021grid, vinci2019graph, dasarathy2019gaussian, ghoroghchian2021graph} or adaptively acquired data \cite{dasarathy2016active}; these variations in-turn may result in significant improvements to the data-requirement even in the setting of this paper. Finally, as is well known, verifying regularity conditions, such as the mutual incoherence,\footnote{ Interestingly, this condition is necessary and sufficient for sparse linear regression problems \citep{wainwright2009sharp}.} in practice is computationally hard. Hence, it would be a worthwhile pursuit to deduce computationally tractable sufficient conditions that are operationally interpretable for systems obeying conservation laws. 

\bibliographystyle{plainnat}
\bibliography{mybibilography} 

%%%%%%%%%%%%%%%%%%%%%%%%%%%%%%%%%%%%%%%%%%%%%%%%%%%%%%%%%%%%
\newpage
\section*{Appendix}
\appendix
% \section{Proof of Technical Results and the Main Theorem}

We begin by giving a brief overview of the problem set-up and provide proofs for all the technical results.
 
% \subsection{Proofs of all technical results}
\noindent{\bf Overview}: We begin with a brief overview of the problem set-up and state the necessary assumptions. Then, we provide proofs for all the technical results. Recall that our observation model is $Y = {B^{*}}^{-1}X$, where $B^{*}$ is a $p\times p$ sparse matrix which encodes the structure of a network with the property that $B^{*}_{ij} = 0$ for all $(i,j)\in E^{c}$, $Y\in \mathbb{R}^{p}$ is the random vector of node potentials and $X\in\mathbb{R}^{p}$ is the unknown vector of injected flows with known covariance matrix $\Sigma_{X}$. Given $n$ i.i.d samples of the vector $Y$ our goal is to learn the sparsity structure of the matrix $B^{*}$. Towards this we propose an estimator $\widehat{B}$ which is the solution of the following $\ell_{1}$ regularized log-det problem
\begin{align}\label{appeq: log-det problem}
\widehat{B} = \argmin_{B\succ0}\left[\Tr(DBSBD)-\log\det(B^2)+\lambda_{n}\Vert B\Vert_{1,\text{off}}\right].
\end{align}
where $D\in \mathbb{R}^{p\times p}$ is the unique square root of $\Sigma_{X}^{-1}$ and $S$ is the sample covariance matrix constructed from $n$ samples of the random vector $Y$. We recall the assumptions necessary to prove our results.\\

\textbf{[A1] Mutual incoherence condition.}  Let $\Gamma^{*}$ be the Hessian of the log-determinant function in \eqref{appeq: log-det problem}: 
\begin{align}\label{appeq: hessian log-det}
    \Gamma^{*} \triangleq \nabla^{2}_{B}\log\det(B)\vert_{B=B^{*}} = {{B^{*}}^{-1}}\otimes {{B^{*}}^{-1}}. 
\end{align}
For $\Gamma^{*}$ in \eqref{appeq: hessian log-det}, there exists some $\alpha \in (0,1]$ such that 
    $\vertiii{\Gamma^{*}_{E^{c}E}(\Gamma^{*}_{EE})^{-1}}_{\infty}\leq 1-\alpha$. 

\textbf{[A2] Hessian regularity condition.} Let $d$ be the maximum number of non zero entries among all the rows in $B^*$ (i.e., the degree of the underlying graph), $\Theta^*=B^*\Sigma^{-1}_{X}B^*$, and $D^2=\Sigma^{-1}_{X}$. Then, 
\begin{align}\label{appeq: Hessian regularity}
    \vertiii{{\Gamma^{*}}^{-1}}_{\infty}\leq \frac{1}{4d\Vert {\Theta^{*}}^{-1}\Vert_{\infty}\vertiii{D^{2}}_{\infty}}.
\end{align}

\textbf{[A3] Maximum row norm condition.}\label{app: max row norm} There exists a constant $c>0$ for which 
$\vertiii{B^{*}}_{\infty}\geq c$, or equivalently, the spectral norm is bounded as $ \vertiii{B^{*}}_{2} \geq c/\sqrt{p}$.\\

Our analysis is based on the Primal-Dual Witness (PDW) construction to certify the behaviour of the estimator $\widehat{B}$. The PDW technique  consists of constructing a primal-dual pair $(\widetilde{B},\widetilde{Z})$, where $\widetilde{B}$ is the primal solution of the restricted log-det problem defined below
\begin{align}\label{appeq: rest log-det}
  \widetilde{B} \triangleq \argmin_{B=B^{T},B\succ 0,B_{E^{c}}=0}\left[\Tr(DBSBD) - \log\det(B^{2})+\lambda_{n}\Vert B\Vert_{1,\text{off}}\right].   
\end{align}
where $\widetilde{Z}$ is the optimal dual solution. By definition the primal solution $\widetilde{B}$ satisfies $\widetilde{B}_{E^{c}} = B^{*}_{E^{c}}=0$. Furthermore the pair $(\widetilde{B},\widetilde{Z})$ are solutions to the zero gradient conditions of the restricted problem $\eqref{appeq: rest log-det}$. Therefore, when the PDW construction succeeds the solution $\widehat{B}$ is equal to the primal solution $\widetilde{B}$ which guarantees the support recovery property ie. $\widehat{B}_{E^{c}}=0$.\\

\pgfkeys{/csteps/inner color = white}
\pgfkeys{/csteps/outer color = black}
\pgfkeys{/csteps/fill color = black}

We now summarize our technical results. \Circled{1} We begin by showing that the $\ell_{1}$ regularized log-det problem in \eqref{appeq: log-det problem} is convex and admits a unique solution $\widehat{B}$ (see Lemma \ref{app: convexity and uniqueness}). 
\Circled{2} We then proceed to derive the sufficient conditions under which the PDW construction succeeds (see Lemma \ref{app: sufficient condition}).  
\Circled{3} We then guarantee that the remainder term $R(\Delta)$ is bounded if $\Delta$ is bounded (see Lemma \ref{app: control remainder}). 
\Circled{4} Furthermore, for a specific choice of radius $r$ as a function of $\Vert W\Vert_{\infty}$we show that $\Delta$ lies in a ball $\mathbb{B}_{r}$ of radius $r$ (see Lemma \ref{app: control of delta}). 
\Circled{5} We then derive a lemma which we call the master lemma which gives support recovery guarantees and element-wise $\ellinf$ norm consistency for our estimator $\widehat{B}$ under no specific distributional assumptions (see Lemma \ref{app: master lemma}). 
\Circled{6} Using known concentration results  on sub-gaussian and moment bounded random vectors we prove our main result for the two distributions mentioned above. Recall that our main result gives sufficient conditions on the number of samples required for our estimator $\widehat{B}$ to recover the exact sparsity structure of $B^{*}$. We also show that under these sufficient conditions $\widehat{B}$ is consistent with $B^{*}$ in the element-wise $\ellinf$ norm and achieves sign consistency if $\vert B^{*}_{\min}\vert$ (the minimum non-zero entries of $B^{*}$) is lower bounded (see Theorem \ref{app: sub-gaussian support recovery} and Theorem~\ref{app: bounded moment support recovery}). 
\Circled{7} Finally, we show that $\widehat{B}$ is consistent in the Frobenius and spectral norm.

\noindent\textbf{Numbering convention}: To make the Appendix self contained we restated statements of all theorems, lemmas, and definitions with their numbers unchanged. For the numbered environments that are specifically introduced in Appendix, the environment begins with the label "A" (e.g., Lemma A.1).

\setcounter{lemma}{0}
\begin{lemma}{{\emph(Convexity and uniqueness)}}\label{app: convexity and uniqueness}
For any $\lambda_{n}\!>\!0$ and $B\!\succ\! 0$, (i) the $\ell_{1}$-log determinant problem in \eqref{appeq: log-det problem} is convex and (ii) $\widehat{B}$ in \eqref{appeq: log-det problem} is the unique minima satisfying the sub-gradient condition $2D^{2}\widehat{B}S - 2\widehat{B}^{-1}\!+\!\lambda_{n}\widehat{Z}\!=\!0$. Here
% \begin{align}
%     2D^{2}\widehat{B}S - 2\widehat{B}^{-1} + \lambda_{n}\widehat{Z} = 0
% \end{align}
$\widehat{Z}$ belong to the sub-gradient $\partial\Vert \widehat{B} \Vert_{1,\text{off}}$ such that $\widehat{Z}_{ij}=0$, for $i=j$, $\widehat{Z}_{ij}\!=\!\mathrm{sign}(\widehat{B}_{ij})$ when $\widehat{B}_{ij}\neq 0$ and $|\widehat{Z}_{ij}|\leq 1$ when $\widehat{B}_{ij}=0$, for $i \ne j$.  

\end{lemma}
\begin{proof}
\emph{(i) Convexity}: Let $S=MM^T$, for some $M\succ 0$, and recall that $\|A\|_F^2=\mathrm{Tr}(AA^\transpose)$. Then, the objective function in \eqref{appeq: log-det problem} can be expressed as 
\begin{align}\label{eq: rest log-det1}
\|DBM\|_F^2 - \log\det(B^{2})+\lambda_{n}\Vert B\Vert_{1,\text{off}}. 
\end{align}
First, the square-root of the first term is convex because for any $\lambda\in (0,1)$ and $B_1,B_2 \succ 0$, we have 
\begin{align*}
 \|D(\lambda B_1+(1-\lambda)B_2)M\|_F&=\|\lambda DB_1M+(1-\lambda)DB_2M\|_F\\
 & \leq \lambda \|DB_1M\|_F+(1-\lambda)\|DB_1M\|_F. 
\end{align*}
Now that $h_1(x)=x^2$ and $h_2(A)=\|A\|_F$ are both convex and that $h_1(x)$ is non-decreasing on the range of $h_2$,  that is, $[0,\infty]$, it follows that the composition $h_1\circ h_2=\|\cdot\|_F^2$ is convex.

Second, we show the convexity of $- \log\det(B^{2})$ using the perspective function technique \citep{Boyd}. To this end, let $|\cdot|$ be the absolute value and note that $\log\det(B^2)=\log|\det(B^2)|=2\log|\det(B)|$. Let $g(t)=\log|\det(B+tV)|$ with $V\succeq 0$ be the perspective function of $\log|\det(B)|$. Since $B$ is symmetric and invertible, there exists an orthogonal matrix $Q$ such that $QQ^\transpose=I$ and $B=Q\Lambda Q^\transpose$, where $\Lambda$ is a diagonal matrix consisting of eigenvalues of $B$. Then, 
\begin{align}\label{eq: convex opt app}
    g(t)&=\log(|\det(Q\Lambda Q^\transpose+tQQ^\transpose VQQ^\transpose)|)\\
    &=\log(|\det(Q(\Lambda+tQ^\transpose VQ)Q^\transpose)|)\\
    &=\log(|\det(\Lambda+tQ^\transpose VQ)|)\\
    &=\log(|\det(I+t\Lambda^{-1}Q^\transpose VQ)|)+\log(|\Lambda|),
\end{align}
where we used the facts $|\det(X_1X_2)|=|\det(X_1)\det(X_2)|=|\det(X_1)||\det(X_2)|$ and $\Lambda$ is full rank. Since $\Lambda$ is diagonal and $Q^\transpose V Q\succeq 0$, it follows that the eigenvalues $\{\lambda_i\}$ of ${\Lambda}^{-1}Q^\transpose VQ$ are real-valued (need not be positive). Thus, 
\begin{align*}
    g(t)&=\log\prod_{i}|(1+t\lambda_i)|+\log(|\Lambda|). 
\end{align*}
Notice that $g'(t)=\sum\frac{\lambda_i(1+t\lambda_i)}{(1+t\lambda_i)^2}$ and $g''(t)=-\sum\frac{\lambda_i^2}{(1+t\lambda_i)^2}<0$. Thus $g(t)$ is strictly concave. Hence, $2\log|\det(B)|=\log\det(B^2)$ is strictly concave. Finally, $-\log\det(B^2)$ is strictly convex. 

Third, the norm $\lambda_n\|B\|_{1,\mathrm{off}}$ is the sum of absolute values of off-diagonal terms, and hence, convex. Because the sum of convex functions and a strictly convex function is strictly convex, we conclude that the objective function in \eqref{eq: convex opt app} is strictly convex. 

{\bf Remark}: In the proof, we used the fact that $B$ is symmetric and full rank but not the positive-definite. The proof for $B\succ 0$ is simple because we can drop the absolute values and mimic the standard log-det concavity proof \cite{ravikumar2011high, friedman2008sparse}. Finally, we required $V\succeq 0$ to not to deal with the (possible) imaginary eigenvalues of ${\Lambda}^{-1}Q^\transpose VQ$. However, we conjecture that $V$ needs to be only symmetric.

\emph{(ii) Uniqueness}: In part (i), we showed that the objective function in \eqref{eq: convex opt app} is strictly convex. Recall that strictly convex functions have the property that the minimum is unique if attained \citep{Boyd}. We show that the minimum is attained using the notion of coercivity \citep{bauschke2011convex}. This amounts to showing that the objective function $CO(B) = (\Vert DBM\Vert_{F}^{2}-2\log\det B)$ subject to constraints (see below) tend to infinity as $\|B\|_2\to \infty$. 

By Lagrangian duality, the $\ell_{1}$ regularized log-det problem \eqref{appeq: log-det problem} can be written as 
\begin{align}
    \argmin_{B\succ0,B=B^{T},\Vert B\Vert_{1,\text{off}}<\mathcal{C}_{n}} \Vert DBM\Vert_{F}^{2}-2\log\det B, 
\end{align}
where $\mathcal{C}_{n}$ is the constraint on the off diagonal elements of $B$.  
From the constraint $\Vert B\Vert_{1,\text{off}}<\mathcal{C}_{n}$, it follows that the off-diagonal elements of $B$ lie in an $\ell_{1}$ ball. Thus, $\Vert B\Vert_{2}\to \infty$ if and only if for any sequence of diagonal elements 
$\Vert\left[ B_{11},\ldots,B_{pp}\right]\Vert_{\infty}\to \infty$. On the other hand, by Hadamard's inequality for positive definite matrices \citep{horn2012matrix}, we have $2\log\det B\leq \sum_{i}2\log B_{ii}$. Thus, 
\begin{align}\label{eq: const opt lower bound}
    \Vert DBM\Vert_{F}^{2}-2\log\det B\geq \Vert DBM\Vert_{F}^{2}-2\sum_{i}\log B_{ii}. 
\end{align} 
We lower bound $\Vert DBM\Vert_{F}^{2}$ as follows. 
For any two compatible matrices $P$ and $Q$, 
\begin{align}\label{eq: F-norm product bound}
    \Vert PQ\Vert_{F}^{2} = \sum_{j}\Vert Pq_{j}\Vert^{2}\geq (\sigma_{\min}(P))^{2}\sum_{j}\Vert q_{j}\Vert^{2}=(\sigma_{\min}(P))^{2}\Vert Q\Vert_{F}^{2}, 
\end{align}
where $q_{j}$ are the columns of $Q$ and $\sigma_{\min}$ is the minimum singular value. Thus, 
\begin{align}
    \Vert DBM\Vert_{F}^{2}\geq (\sigma_{\min}(D))^{2}(\sigma_{\min}(M^{T}))^{2}\Vert B\Vert_{F}^{2}. 
\end{align}
Ignoring the off diagonal elements of $B$ (because they are bounded), from \eqref{eq: const opt lower bound} and \eqref{eq: F-norm product bound}, we get 
\begin{align}
  \Vert DBM\Vert_{F}^{2}-2\log\det B\geq  (\sigma_{\min}(D))^{2}(\sigma_{\min}(M^{T}))^{2}\left[\sum_{i}B_{ii}^{2}\right] -2\sum_{i}\log B_{ii}. 
\end{align}
Since the second term in the lower bound is logarithmic, the first term in the lower bound dominates it for large $B^2_{ii}$. Consequently, the lower bound, and hence $CO(B)$, tend to infinity as $\|B\|_2\to \infty$. Therefore $O(B)$ is coercive and the minimum is attained and it is unique.
\end{proof}

We derive sufficient conditions under which the PDW construction (defined in Section 3.2) succeeds. 
\begin{lemma}(Sufficient conditions for strict dual feasibility)\label{app: sufficient condition} Let the regularization parameter $\lambda_n>0$ and $\alpha$ be defined as in {\bf [A1]}. Suppose the following holds 
\begin{align}\label{eq: sufficient condition}
    \max\left\lbrace\vertiii{\Gamma(D^{2}\Delta)+\Gamma(D^{2}B^{*})}_{\infty}\Vert W\Vert_{\infty},\Vert R(\Delta)\Vert_{\infty},\vertiii{\Gamma(D^{2}\Delta)}_{\infty}\Vert {\Theta^{*}}^{-1}\Vert_{\infty}\right\rbrace\leq \frac{\lambda_{n}\alpha}{24}. 
\end{align}
Then the dual vector $\widetilde{Z}_{E^{c}}$ satisfies $\Vert \widetilde{Z}_{E^{c}}\Vert_{\infty}<1$, and hence, $\widetilde{B} = \widehat{B}$. 
\end{lemma}
\begin{proof}

We begin by obtaining a suitable expression for $\widetilde{Z}_{E^{c}}$ using the zero-subgradient condition of the the restricted $\ell_{1}$ regularized log-det problem defined in \eqref{appeq: rest log-det}: 
\begin{align}\label{eq: rest log-det1}
    \widetilde{B} = \argmin_{B=B^{T},B\succ 0,B_{E^{c}}=0}\left[\Tr(DBSBD) - \log\det(B^{2})+\lambda_{n}\Vert B\Vert_{1,\text{off}}\right]. 
\end{align}
The zero-subgradient of the restricted problem is given by
\begin{align}
 2D^{2}\widetilde{B}S - 2\widetilde{B}^{-1} + \lambda_{n}\widetilde{Z} = 0, 
\end{align}
where $\widetilde{B}$ is the primal solution given by \eqref{eq: rest log-det1} and $\widetilde{Z}\in \partial \Vert B\Vert_{1,\text{off}}$ is the optimal dual. Recall that $\Delta = \widetilde{B}-B^{*}$ and $W = S-{\Theta^{*}}^{-1}$ and notice the following chain of identities: 
\begin{align*}
    2(D^{2}\widetilde{B}S - \widetilde{B}^{-1}) + \lambda_{n}\widetilde{Z} &=
    2(D^{2}\widetilde{B}S - D^{2}B^{*}S + D^{2}B^{*}S - \widetilde{B}^{-1}) + \lambda_{n}\widetilde{Z}\\
    &=2(D^{2}\Delta S + D^{2}B^{*}S - \widetilde{B}^{-1}) + \lambda_{n}\widetilde{Z}\\
    &=2(D^{2}\Delta W + D^{2}B^{*}W+ D^{2}\Delta{\Theta^{*}}^{-1}+D^{2}B^{*}{\Theta^{*}}^{-1}-\widetilde{B}^{-1})+\lambda_{n}\widetilde{Z}. 
\end{align*}
On the other hand, by definition, ${\Theta^{*}}^{-1} = {B^{*}}^{-1}\Sigma_{X}{B^{*}}^{-1}$ and $(D^{2})^{-1} = \Sigma_{X}$. Thus,  $D^{2}B^{*}{\Theta^{*}}^{-1} = {B^{*}}^{-1}$. Substituting these expressions in the zero-subgradient condition yields the following:
\begin{align}\label{eq: app_lemma1_sub-grad1}
    D^{2}\Delta W + D^{2}B^{*}W+ D^{2}\Delta{\Theta^{*}}^{-1}+{B^{*}}^{-1}-\widetilde{B}^{-1} + \lambda_{n}^{\prime}\widetilde{Z} = 0,
\end{align}
where $\lambda_{n}^{\prime}=
    0.5\lambda_{n}$. By adding and subtracting ${B^{*}}^{-1}\Delta{B^{*}}^{-1}$ to the preceding equality and followed by some algebraic manipulations give us 
\begin{align}
 {B^{*}}^{-1}\Delta{B^{*}}^{-1}+ D^{2}\Delta W + D^{2}B^{*}W+ D^{2}\Delta{\Theta^{*}}^{-1}-R(\Delta)+\lambda_{n}^{\prime}\widetilde{Z} = 0, 
\end{align}
where $R(\Delta) = \widetilde{B}^{-1}-{B^{*}}^{-1} - {B^{*}}^{-1}\Delta {B^{*}}^{-1}$.

We now vectorize \eqref{eq: app_lemma1_sub-grad1}. {We use $\mvec(A)$ or $\bar{A}$ to denote the $p^2$-vector formed by stacking the columns of $A$ and use $\Gamma(A)=(I\otimes A)$ to denote the Kronecker product of $A$ with the identity matrix $I$.} By applying $\mvec()$ operator on both sides of \eqref{eq: app_lemma1_sub-grad1} it follows that 
\begin{align}\label{eq: vector stationary}
    \mvec({B^{*}}^{-1}\Delta{B^{*}}^{-1}+ D^{2}\Delta W + D^{2}B^{*}W+ D^{2}\Delta{\Theta^{*}}^{-1}-R(\Delta)+\lambda_{n}^{\prime}\widetilde{Z}) = 0.
\end{align}
Using the standard Kronecker matrix product rules \citep{laub2005matrix}, we have $\mvec({B^{*}}^{-1}\Delta{B^{*}}^{-1}) = \Gamma^{*}\xbar{\Delta}$ and $\mvec((D^{2}\Delta)W) = \Gamma(D^{2}\Delta)\xbar{W}$, where $\Gamma^{*} = {B^{*}}^{-1}\otimes{B^{*}}^{-1}$; $\Gamma(D^{2}\Delta) = I\otimes D^{2}\Delta$; and $I$ is the $p\times p$ identity matrix. By substituting these observations in \eqref{eq: vector stationary}, we note that
%{B^{*}}^{-1}\otimes{B^{*}}^{-1}
\begin{align}
    \Gamma^{*}\xbar{\Delta} + \Gamma(D^{2}\Delta)\xbar{W} + \Gamma(D^{2}B^{*})\xbar{W} + \Gamma(D^{2}\Delta)\xbar{{\Theta^{*}}^{-1}}-\xbar{R(\Delta)}+\lambda_{n}^{\prime}\xbar{\widetilde{Z}} = 0. 
\end{align}
For compactness, we suppress $\Delta$ notation in $R(\Delta)$. Recall that the $E$ is the augmented set defined as $E:=\{\mathcal{E}(B^{*})\cup (1,1)\ldots\cup (p,p)\}$, where $\mathcal{E}$ is the edge set of the network and $E^{c}$ is the complement of the set $E$. Recall that we use the notation $A_{E}$ to denote the sub-matrix of $A$ containing all elements $A_{ij}$ such that $(i,j)\in E$. We partition the preceding linear equations into two separate linear equations corresponding to the sets $E$ and $E^{c}$ as 
 \begin{align}
 \Gamma_{EE}^{*}\xbar{\Delta}_{E} + \left(\Gamma_{EE}(D^{2}\Delta) + \Gamma_{EE}(D^{2}B^{*})\right)\xbar{W}_{E} + \Gamma_{EE}(D^{2}\Delta)\xbar{{\Theta_{E}^{*}}^{-1}}-\xbar{R}_{E}+\lambda_{n}^{\prime}\xbar{\widetilde{Z}}_{E} &= 0, \label{eq: vectoreq1 }\\
    \Gamma_{{E^{c}}E}^{*}\xbar{\Delta}_{E} + \left(\Gamma_{E^{c}E}(D^{2}\Delta) + \Gamma_{E^{c}E}(D^{2}B^{*})\right)\xbar{W}_{E^{c}} + \Gamma_{E^{c}E}(D^{2}\Delta)\xbar{{\Theta_{E^{c}}^{*}}^{-1}}-\xbar{R}_{E^{c}}+\lambda_{n}^{\prime}\xbar{\widetilde{Z}}_{E^{c}} &= 0. \label{eq: vectoreq2}
\end{align}
From \eqref{eq: vectoreq1 }, we can solve for $\xbar{\Delta}_{E}$ as
\begin{align}\label{eq: define M}
    \xbar{\Delta}_{E} = (\Gamma_{EE}^{*})^{-1}\underbrace{\left[-\left(\left(\Gamma_{EE}(D^{2}\Delta) + \Gamma_{EE}(D^{2}B^{*})\right)\xbar{W}_{E}+\Gamma_{EE}(D^{2}\Delta)\xbar{{\Theta_{E}^{*}}^{-1}}\right)+\xbar{R}_{E}-\lambda_{n}^{\prime}\xbar{\widetilde{Z}}_{E}\right]}_{\triangleq M}. 
\end{align}
Substituting $\xbar{\Delta}_{E}$ given by \eqref{eq: define M} in \eqref{eq: vectoreq2} gives us 
\begin{equation}
   \Gamma_{{E^{c}}E}^{*}(\Gamma_{EE}^{*})^{-1}M + \left(\Gamma_{E^{c}E}(D^{2}\Delta) + \Gamma_{E^{c}E}(D^{2}B^{*})\right)\xbar{W}_{E^{c}}\Gamma_{E^{c}E}(D^{2}\Delta)\xbar{{\Theta_{E^{c}}^{*}}^{-1}}-\xbar{R}_{E^{c}}+\lambda_{n}^{\prime}\xbar{\widetilde{Z}}_{E^{c}}\!=\!0. 
\end{equation}
From which we can solve for the vectorized dual $\xbar{\widetilde{Z}}_{E^{c}}$ as 
\begin{equation}
    \begin{split}
     \lambda_{n}^{\prime}\xbar{\widetilde{Z}}_{E^{c}}\!=\! -\Gamma_{{E^{c}}E}^{*}(\Gamma_{EE}^{*})^{-1}M - \left(\Gamma_{E^{c}E}(D^{2}\Delta) + \Gamma_{E^{c}E}(D^{2}B^{*})\right)\xbar{W}_{E^{c}}-\Gamma_{E^{c}E}(D^{2}\Delta)\xbar{{\Theta_{E^{c}}^{*}}^{-1}} +\xbar{R}_{E^{c}}. 
    \end{split}
\end{equation}
Taking the element-wise $\ellinf$ norm on both sides of the preceding equality gives us 
\begin{equation}
    \begin{split}
        \Vert \xbar{\widetilde{Z}}_{E^{c}}\Vert_{\infty} &\leq  \frac{1}{\lambda_{n}^{\prime}}\vertiii{\Gamma_{{E^{c}}E}^{*}(\Gamma_{EE}^{*})^{-1}}_{\infty}\Vert M\Vert_{\infty} + \frac{1}{\lambda_{n}^{\prime}}\vertiii{\Gamma_{E^{c}E}(D^{2}\Delta)}_{\infty}\Vert \xbar{{\Theta_{E^{c}}^{*}}^{-1}}\Vert_{\infty}\\
        &+\frac{1}{\lambda_{n}^{\prime}}\vertiii{\Gamma_{E^{c}E}(D^{2}\Delta) + \Gamma_{E^{c}E}(D^{2}B^{*})}_{\infty}\Vert{\xbar{W}_{E^{c}}}\Vert_{\infty} + \frac{1}{\lambda_{n}^{\prime}}\vertiii{\xbar{R}_{E^{c}}}_{\infty}. 
    \end{split}
\end{equation}
We invoke the mutual incoherence condition in \eqref{appeq: hessian log-det} to bound $\vertiii{\Gamma_{{E^{c}}E}^{*}(\Gamma_{EE}^{*})^{-1}}_{\infty}\leq (1-\alpha)$ and since $\Vert A_{E^{c}}\Vert_{\infty}\leq \Vert A\Vert_{\infty}$ for any matrix $A$, we get
\begin{equation}
    \begin{split}\label{eq: modifiedsdf}
       \Vert \xbar{\widetilde{Z}}_{E^{c}}\Vert_{\infty} &\leq\frac{1-\alpha}{\lambda_{n}^{\prime}}\Vert M\Vert_{\infty}+ \frac{1}{\lambda_{n}^{\prime}}\vertiii{\Gamma(D^{2}\Delta)}_{\infty}\Vert {\Theta^{*}}^{-1}\Vert_{\infty}\\
        & +\frac{1}{\lambda_{n}^{\prime}}\left[\vertiii{\Gamma(D^{2}\Delta) + \Gamma(D^{2}B^{*})}_{\infty}\Vert W\Vert_{\infty} + \vertiii{R}_{\infty}\right]. 
    \end{split}
\end{equation}
We bound $\Vert M\Vert_{\infty}$ by taking the element-wise $\ellinf$ norm of $M$ in \eqref{eq: define M} and followed by applying sub-multiplicative norm inequalites. Thus,
\begin{equation}
    \begin{split}
        \Vert M\Vert_{\infty}&\leq \vertiii{\Gamma_{EE}(D^{2}\Delta) + \Gamma_{EE}(D^{2}B^{*})}_{\infty}\Vert W_{E}\Vert_{\infty} +  \vertiii{\Gamma_{EE}(D^{2}\Delta)}_{\infty}\Vert {\Theta_{E}^{*}}^{-1}\Vert_{\infty}\\
        &+ \vertiii{R_{E}}_{\infty} + \lambda_{n}^{\prime}\Vert \widetilde{Z}_{E}\Vert_{\infty}. 
    \end{split}
\end{equation}
Because $\widetilde{Z}_{E}$ is the sub-vector of the vectorized optimal dual $\widetilde{Z}$, it follows that $\Vert\widetilde{Z}_{E}\Vert_{\infty}\leq 1$. Thus,
\begin{equation}\label{eq: infinity M bound final}
    \begin{split}
        \Vert M\Vert_{\infty}&\leq \underbrace{\left[\vertiii{\Gamma(D^{2}\Delta) + \Gamma(D^{2}B^{*})}_{\infty}\Vert W\Vert_{\infty} +  \vertiii{\Gamma(D^{2}\Delta)}_{\infty}\Vert {\Theta^{*}}^{-1}\Vert_{\infty}+ \vertiii{R}_{\infty}\right]}_{\triangleq H}+ \lambda_{n}^{\prime}. 
    \end{split}
\end{equation}
On the other hand, from \eqref{eq: infinity M bound final}, we have $H\leq \lambda'_n\alpha/4$ from assumption in Lemma \ref{app: sufficient condition}. Putting together the pieces, from \eqref{eq: infinity M bound final} and \eqref{eq: modifiedsdf} we conclude that 
\begin{align}
   \Vert \xbar{\widetilde{Z}}_{E^{c}}\Vert_{\infty} &\leq (1-\alpha) + \frac{1-\alpha}{\lambda_{n}^{\prime}}H + \frac{1}{\lambda_{n}^{\prime}} H\\
   &= (1-\alpha) + \frac{2-\alpha}{\lambda_{n}^{\prime}}H\\
   &\leq (1-\alpha) + \frac{2-\alpha}{\lambda_{n}^{\prime}}\left(\frac{\lambda_{n}^{\prime}\alpha}{4}\right)\\
   & \leq (1-\alpha) +\frac{\alpha}{2}<1. 
\end{align}

{\bf Remark} For comparison, consider the strict dual feasibility conditions in \cite[Lemma 4]{ravikumar2011high}. Here, the maximum is on the noise deviation $\|W\|_\infty$ and the remainder term $\|R(\Delta)\|_\infty$. Instead, in our case, the maximum is taken over several other quantities not just $\|W\|_\infty$ and $\|R(\Delta)\|_\infty$ (see \eqref{app: sufficient condition}). 
\end{proof}
The following lemma shows that the remainder term $R(\Delta)$ is bounded if $\Delta$ is bounded. The proof is adapted from \citep{ravikumar2011high}, where a similar result is derived using matrix expansion techniques. We use this lemma in the proof of our main result (see Theorem \ref{app: sub-gaussian support recovery} and Theorem \ref{app: bounded moment support recovery}) to show that with sufficient number of samples $R(\Delta)\leq \alpha\lambda_{n}/24$.
\begin{lemma}({Control of reminder})\label{app: control remainder}  Suppose that the element-wise $\ellinf$-bound $\Vert \Delta\Vert_{\infty}\leq \frac{1}{3\nu_{{B^{*}}^{-1}}d}$ holds, then the matrix $Q = \sum\limits_{k=0}^{\infty}(-1)^{k}({B^{*}}^{-1}\Delta)^{k}$ satisfies the bound $\nu_{Q^{T}}\leq \frac{3}{2}$ and the matrix $R(\Delta) = {B^{*}}^{-1}\Delta{B^{*}}^{-1}\Delta Q {B^{*}}^{-1}$ has the element-wise $\ellinf$-norm bounded as 
 \begin{align}
     \Vert R(\Delta)\Vert_{\infty}\leq \frac{3}{2}d\Vert \Delta\Vert^{2}_{\infty}\nu_{{B^{*}}^{-1}}^{3}.
\end{align}
 \end{lemma}
We show that for a specific choice of radius $r$, the distortion $\Delta = \widetilde{B}-B^{*}$ lies in a ball of radius $r$.
\begin{lemma}({Control of $\Delta$})\label{app: control of delta}
 Let
 \begin{align*}
     r\!\triangleq\!4\nu_{{\Gamma^{*}}^{-1}}\left[\nu_{D^{2}}\nu_{B^{*}}\Vert W\Vert_{\infty}\!+\!0.5{\lambda_{n}}\right]\leq \min\Big\{\frac{1}{3\nu_{{B^{*}}^{-1}}d},\frac{1}{6\nu_{{\Gamma^{*}}^{-1}}\nu_{{B^{*}}^{-1}}^{3}d}\Big\}. 
 \end{align*}
Then we have the element-wise $\ellinf$ bound $\Vert \Delta\Vert_{\infty} = \Vert \widetilde{B}-B^{*}\Vert_{\infty}\leq r$.
\end{lemma}
\begin{proof} We adopt the proof technique in \citep[Lemma 6]{ravikumar2011high}. We use the notation $A_{E}$ or $[A]_E$ to denote the sub-matrix of $A$ containing all elements $A_{ij}$ such that $(i,j)\in E$. Let $G(\widetilde{B}_{E})$ be the zero sub-gradient condition of the restricted $\ell_{1}$ log-det problem in \eqref{appeq: rest log-det}: 

\begin{align}
    G(\widetilde{B}_{E}) = \left[D^{2}\widetilde{B}S - \widetilde{B}^{-1}+\lambda_{n}^{\prime}\widetilde{Z}\right]_{E} = 0.
\end{align}
where $\lambda_{n}^{\prime} = 0.5 \lambda_{n}$. Let $\xbar{G}$ denote the vectorized form of $G$. Recall that $\Delta=\widetilde{B}-B^{*}=\widetilde{B}_{E}-B^{*}_{E}\triangleq \Delta_E$. The second equality follows from PDW construction and the constraint in the restricted convex program in \eqref{eq: rest log-det}. To establish $\Vert \Delta\Vert_{\infty}\leq r$, we show that $\Delta_E$ lies inside the ball $\mathbb{B}_{r} = \{\xbar{A}_{E}\in \mathbb{R}^{\vert E\vert}: \Vert A\Vert_{\infty}\leq r\}$, where $\xbar{A}_{E}=\mvec{}(A_E)$, using a contraction property of the continuous map: 
\begin{align}\label{eq: F cmap}
    F(\xbar{\Delta}_{E}) \triangleq -(\Gamma_{EE}^{*})^{-1}\left(\xbar{G}(\Delta_{E}+B_{E}^{*})\right)+\xbar{\Delta}_{E}, 
\end{align}
where we used the fact that $\widetilde{B}_{E^c}=\widetilde{B}^*_{E^c}=0$. 

Suppose that $F(\cdot)$ is a contraction on $\mathbb{B}_{r}$, i.e., $F(\mathbb{B}_{r})\subseteq \mathbb{B}_{r}$. Then by Brower's fixed point theorem \citep{kellogg}, it readily follows that there exists a $C\in \mathbb{B}_{r}$ such that $F(C) = C$. Finally, ${C}=\xbar{\Delta}_{E}$ because
(i) $\widetilde{B}$ that satisfies $\xbar{G}(\widetilde{B})=0$ is unique (see Lemma \ref{lma: uniq soln}) and (ii) $F(\xbar{\Delta}_{E})=\xbar{\Delta}_{E}$ if and only if $\xbar{G}(\cdot) = 0$, Hence, $\xbar{\Delta}_{E}\in \mathbb{B}_{r}$ is the unique fixed point of $F(\cdot)$ in \eqref{eq: F cmap}. Consequently, $\|\Delta_E\|_\infty\leq r$. 

It remains to show that $F(\cdot)$ is a contraction. Let $\Delta^{\prime}\in \mathbb{R}^{p\times p}$ be a zero padded matrix on $E^{c}$ such that $\xbar{\Delta^{\prime}}_{E}\in \mathbb{B}_{r}$. Then $F(\xbar{\Delta^{\prime}}_{E})$ can be expanded in terms of $\Delta'$ as 
\begin{equation}
    \begin{split}
        F(\xbar{\Delta^{\prime}}_{E}) &= -(\Gamma_{EE}^{*})^{-1}\left(\xbar{G}(\Delta_{E}^{\prime}+B_{E}^{*})\right)+ \xbar{\Delta_{E}^{\prime}}\\
    &=-(\Gamma_{EE}^{*})^{-1}\left[ \mvec([D^{2}(\Delta^{\prime}+B^{*})S]_{E} - (\Delta^{\prime}+B^{*})_{E}^{-1} + \lambda_{n}^{\prime}\widetilde{Z}_{E}) + \Gamma_{EE}^{*}\xbar{\Delta^{\prime}}_{E}\right].
    \end{split}
\end{equation}
Adding and subtracting ${\Theta^{*}}^{-1}$ and ${B_{E}^{*}}^{-1}$ to the preceding equality yields us
\begin{equation}
    \begin{split}\label{eq: contraction}
      F(\xbar{\Delta^{\prime}}_{E}) =& -(\Gamma_{EE}^{*})^{-1}\left[\mvec\left(\left[D^{2}(\Delta^{\prime}+B^{*})W\right]_{E}+\left[D^{2}(\Delta^{\prime}+B^{*}){\Theta^{*}}^{-1}\right]_{E}+\lambda_{n}^{\prime}\widetilde{Z}_{E}-{B_{E}^{*}}^{-1}\right)\right]\\
      & -(\Gamma_{EE}^{*})^{-1}\left[-\mvec\left((\Delta^{\prime}+B^{*})^{-1}-{B_{E}^{*}}^{-1}\right)+\Gamma_{EE}^{*}\xbar{\Delta}_{E}^{\prime}\right].  
    \end{split}
\end{equation}
The last $\mvec{}()$ term can be even simplified as 
\begin{equation}
    \begin{split}
      \mvec\left((\Delta^{\prime}+B^{*})^{-1}-{B^{*}}^{-1}\right)+\Gamma^{*}\Delta^{\prime} & = \mvec\left((\Delta^{\prime}+B^{*})^{-1}-{B^{*}}^{-1}+({B^{*}}^{-1}\Delta^{\prime}{B^{*}}^{-1})\right)\\
      & = \mvec(R(\Delta^{\prime})). 
    \end{split}
\end{equation}
Substituting this observation in \eqref{eq: contraction} and rearranging the terms gives us
\begin{equation}\label{eq: T terms}
    \begin{split}
      F(\xbar{\Delta^{\prime}}_{E}) =&  -\underbrace{(\Gamma_{EE}^{*})^{-1}\mvec\left[D^{2}B^{*}W\lambda_{n}^{\prime}\widetilde{Z}\right]_{E}}_{\triangleq T_{1}}-\underbrace{(\Gamma_{EE}^{*})^{-1}\mvec\left[D^{2}\Delta^{\prime}W\right]_{E}}_{\triangleq T_{2}}\\
      &-\underbrace{(\Gamma_{EE}^{*})^{-1}\left[\xbar{R}(\Delta^{\prime})\right]_{E}}_{\triangleq T_{3}}-\underbrace{(\Gamma_{EE}^{*})^{-1} \mvec\left[D^{2}(\Delta^{\prime}+B^{*}){\Theta^{*}}^{-1}-{B^{*}}^{-1}\right]_{E}}_{\triangleq T_{4}}. 
    \end{split}
\end{equation}
We now show that $\Vert F(\xbar{\Delta^{\prime}}_{E})\Vert_{\infty}\leq r$ by bounding $\ell_\infty$ norms of terms $(T_1)$-$(T_4)$. Recall that $\nu_{A} = \vertiii{A}_{\infty}\triangleq \max_{j=1,\ldots,p}\sum_{j=1}^{p}\vert A_{ij}\vert$ and it is sub-multiplicative; that is $\vertiii{AB}_\infty\leq \vertiii{A}_\infty\vertiii{B}_\infty$. Notice that this not the case with the max norm ($\ell_\infty$). 

(i) \emph{upper bound on $\|T_1\|_\infty$}: Consider the following chain of inequalities. 
    \begin{align}
        \Vert T_{1}\Vert_{\infty}& \leq  \vertiii{{\Gamma^{*}}^{-1}}_{\infty}\left\Vert \mvec(D^{2}B^{*}W+\lambda_{n}^{\prime}\widetilde{Z})\right\Vert_{\infty}\nonumber\\
        &= \vertiii{{\Gamma^{*}}^{-1}}_{\infty}\left\Vert\Gamma(D^{2}B^{*})\xbar{W}+\lambda_{n}^{\prime}\xbar{\widetilde{Z}}\right\Vert_{\infty} \nonumber\\
        &\overset{(a)}{\leq} \vertiii{{\Gamma^{*}}^{-1}}_{\infty}\left[\vertiii{\Gamma(D^{2}B^{*})}_{\infty}\Vert W\Vert_{\infty}+\lambda_{n}^{\prime}\right] \nonumber \\ 
        &\overset{(b)}{\leq} \nu_{{\Gamma^{*}}^{-1}}\left[\nu_{D^{2}}\nu_{B^{*}}\Vert W\Vert_{\infty}+\lambda_{n}^{\prime}\right]\overset{(c)}{\leq} \frac{r}{4},
    \end{align}
where (a) follows because $\Vert\xbar{\widetilde{Z}}\Vert_{\infty}\leq 1$ (see Lemma \ref{lma: uniq soln}); (b) follows because $\Gamma(D^{2}B^{*}) = (I\otimes D^{2}B^{*})$, and hence, $\vertiii{\Gamma(D^{2}B^{*})}_{\infty} = \vertiii{D^{2}B^{*}}_{\infty}\leq \vertiii{D^{2}}_{\infty}\vertiii{B^{*}}_{\infty} = \nu_{D^{2}}\nu_{B^{*}}$; and finally, (c) follows from definition of the radius $r$ in Lemma \ref{app: control of delta}.

(ii) \emph{upper bound on $\|T_2\|_\infty$}: For $T_2$ in \eqref{eq: T terms}, consider the following chain of inequalities. 
\begin{align}
\Vert T_{2}\Vert_{\infty} &\leq \nu_{{\Gamma^{*}}^{-1}}\left[\vertiii{\Gamma(D^{2}\Delta)}_{\infty}\Vert W\Vert_{\infty}\right]\nonumber\\
&\leq \nu_{{\Gamma^{*}}^{-1}}\nu_{D^{2}}\vertiii{\Delta^{\prime}}_{\infty}\Vert W\Vert_{\infty}\nonumber\\
&\overset{(a)}{\leq} \nu_{{\Gamma^{*}}^{-1}}\nu_{D^{2}}d\Vert \Delta^{\prime}\Vert_{\infty}\Vert W\Vert_{\infty} \nonumber\\
&\overset{(b)}{\leq} \nu_{{\Gamma^{*}}^{-1}}\nu_{D^{2}}d\Vert \Delta^{\prime}\Vert_{\infty}\left(\frac{r}{4\nu_{{\Gamma^{*}}^{-1}}\nu_{D^{2}}\nu_{B^{*}}}\right)\nonumber \\
&\overset{(c)}{\leq} d\left(\frac{1}{3d\nu_{{B^{*}}^{-1}}}\right)\left(\frac{r}{4\nu_{B^{*}}}\right)\overset{(d)}{\leq} \frac{r}{4}, 
\end{align}
where (a) follows because by construction $\Delta^{\prime}$ has at-most $d$ non-zeros in every row and that $\vertiii{\Delta^{\prime}}_{\infty}\leq d\Vert\Delta^{\prime}\Vert_{\infty}$; (b) follows from the choice of $r = 4\nu_{{\Gamma^{*}}^{-1}}(\nu_{D^{2}}\nu_{B^{*}}\Vert W\Vert_{\infty} + \lambda^{\prime}_{n})$ in Lemma \ref{app: control of delta}, which is lower bounded by  $4\nu_{{\Gamma^{*}}^{-1}}\nu_{D^{2}}\nu_{B^{*}}\Vert W\Vert_{\infty}$, for all $\lambda^{\prime}_{n}\geq 0$. Thus, 
$\Vert W\Vert_{\infty}\leq r/(4\nu_{{\Gamma^{*}}^{-1}}\nu_{D^{2}}\nu_{B^{*}})$; (c) follows because $\Delta^{\prime}$ is a zero-padded matrix of $\Delta$. Hence $\Vert\Delta \Vert=\Vert \Delta^{\prime}\Vert_{\infty}\leq r$, which can be upper bounded by $1/(3d\nu_{{B^{*}}^{-1}})$ in light of the hypothesis in Lemma \ref{app: control of delta}; and finally,
(d) follows because $\nu_{B^{*}}\nu_{{B^{*}}^{-1}}\geq 1$. 

(iii) \emph{upper bound on $\|T_3\|_\infty$}: For $T_3$ in \eqref{eq: T terms}, consider the following chain of inequalities. 
\begin{align}
    \Vert T_{3}\Vert_{\infty} &\leq \nu_{{\Gamma^{*}}^{-1}}\Vert R(\Delta^{\prime})\Vert_{\infty}\\
    &\overset{(a)}{\leq}\frac{3}{2}d\nu_{{\Gamma^{*}}^{-1}}\nu_{{B^{*}}^{-1}}^{3}\Vert \Delta^{\prime}\Vert_{\infty}^{2}\\
    &\overset{(b)}{\leq}\frac{3}{2}d\nu_{{\Gamma^{*}}^{-1}}\nu_{{B^{*}}^{-1}}^{3}r(r)\overset{(c)}{\leq} \frac{r}{4},
\end{align}
where (a) follows because Lemma \ref{app: control remainder} guarantees that $\Vert R(\Delta^{\prime})\Vert_{\infty}\leq (3/2)d\nu_{{B^{*}}^{-1}}^{3}\Vert \Delta^{\prime}\Vert^{2}_{\infty}$ whenever $\Vert \Delta^{\prime}\Vert_{\infty}\leq 1/(3d\nu_{{B^{*}}^{-1}})$. The latter inequality is a consequence of the hypothesis in Lemma \ref{app: control of delta}; (b) is true because by construction $\Delta^{\prime}\in \mathbb{B}_{r}$, and hence, $\Vert\ \Delta^{\prime}\Vert_{\infty}\leq r$; (c) follows by invoking the hypothesis in Lemma \ref{app: control of delta}, where $r$ satisfies $r\leq 1/(6d\nu_{{\Gamma^{*}}^{-1}}\nu_{{B^{*}}^{-1}}^{3})$. 

(iv) \emph{upper bound on $\|T_4\|_\infty$}: The expression of $T_4$ in \eqref{eq: T terms} can be simplified as 
\begin{align}
    T_{4} &= -(\Gamma_{EE}^{*})^{-1}\mvec\left(\left[D^{2}(\Delta^{\prime}+B^{*}){\Theta^{*}}^{-1}-{B^{*}}^{-1}\right]_{E}\right)\nonumber \\
    & = -(\Gamma_{EE}^{*})^{-1}\mvec\left(\left[D^{2}\Delta^{\prime}{\Theta^{*}}^{-1} + D^{2}B^{*}{\Theta^{*}}^{-1}-{B^{*}}^{-1}\right]_{E}\right)\nonumber \\
    &=-(\Gamma_{EE}^{*})^{-1}\mvec\left(\left[D^{2}\Delta^{\prime}{\Theta^{*}}^{-1}\right]_{E}\right)\label{eq: T4 simplify}. 
\end{align}
The last equality follows by observing that $D^{2}B^{*}{\Theta^{*}}^{-1} = {B^{*}}^{-1}$. This can be verified by plugging ${\Theta^{*}}^{-1} = {B^{*}}^{-1}\Sigma_{X}{B^{*}}^{-1}$ and $\Sigma_{X} = (D^{2})^{-1}$ in $D^{2}B^{*}{\Theta^{*}}^{-1}$ and simplifying the resulting expression. By taking the $\ell_\infty$ bound on the both sides of \eqref{eq: T4 simplify} gives us  
\begin{align}
    \Vert T_{4}\Vert_{\infty} &\leq \nu_{{\Gamma^{*}}^{-1}}\vertiii{\Gamma(D^{2}\Delta^{\prime})}_{\infty}\Vert {\Theta^{*}}^{-1}\Vert_{\infty}\\
    &\overset{}{\leq} \nu_{{\Gamma^{*}}^{-1}}\nu_{D^{2}}d\Vert \Delta^{\prime}\Vert_{\infty}\Vert {\Theta^{*}}^{-1}\Vert_{\infty}\\
    &\overset{}{\leq}\nu_{{\Gamma^{*}}^{-1}}\nu_{D^{2}}rd\Vert {\Theta^{*}}^{-1}\Vert_{\infty}
    \overset{(a)}{\leq} \frac{r}{4}, 
\end{align}
where (a) follows by invoking the assumption in \eqref{appeq: Hessian regularity}, and noting that $\Vert {\Theta^{*}}^{-1}\Vert_{\infty}\leq 1/(4\nu_{{\Gamma^{*}}^{-1}}\nu_{D^{2}}d)$. 

Putting together the pieces, from the above calculations, we note that  
\begin{align}
    \Vert F(\Delta^{\prime}_{E})\Vert_{\infty}&\leq \Vert T_{1}\Vert_{\infty} +\Vert T_{2}\Vert_{\infty}+\Vert T_{3}\Vert_{\infty}+\Vert T_{4}\Vert_{\infty}\leq r. 
\end{align}
is a contraction as claimed. This concludes the proof.
\end{proof}

\setcounter{definition}{0}
\renewcommand{\thedefinition}{A.\arabic{definition}}
We borrow the following notion of tail conditions as defined in \citep{ravikumar2011high} to characterize the distribution. We us this characterization to prove our master lemma \ref{app: master lemma} 
\begin{definition}{(Tail condition, \citep{ravikumar2011high})}\label{def: TC}
The random vector $Y$ satisfies the tail condition $\mathcal{T}(f,v_{*})$ if there exist a constant $v_{*}>0$ and a function $f:\mathbb{N}\times (0,\infty)$ such that for any $i,j \in [p]$ and $\delta\in (0,1/v_{*})$: 
\begin{align}
  \mathbb{P}\left[\vert S_{ij} - \Sigma^{*}_{ij}\vert\geq \delta\right]\leq \frac{1}{f(n,\delta)}. 
\end{align}
Furthermore, $f(n,\delta)$ is monotonically increasing in $n$ (or $\delta$) for fixed $\delta$ (or $n$). 
\end{definition}
Both the exponential-type tail $f(n,\delta)=\exp(cn\delta^a)$ and the polynomial-type tail $f(n,\delta)=cn^m\delta{2m}$, where $m$ is an integer and $c,a>0$, satisfy the monontone property in Definition \ref{def: TC}. The following inverse functions associated with $f(n,\delta)$ are needed to prove our sample complexity result: 
\begin{align}\label{appeq: monotonicity}
  n_{f}(\delta,p^{\tau})\defeq \max\{n \vert f(n,\delta)\leq p^{\tau}\}  \text { and } \delta_{f}(n,p^{\tau}) \defeq \max\{\delta\vert f(n,\delta)\leq p^{\tau}\}.
\end{align}
Both the functions are well-defined thanks to the to the monotonicity property of $f(n,\delta)$. Further, if $n>n_{f}(\delta,p^{\tau}) $ for some $\delta>0$ implies that $\delta\geq \delta_{f}(n,p^{\tau})$.

The following result presents an exponential-type tail bound for sub-Gaussian random vectors.
\setcounter{lemma}{0}
\renewcommand{\thelemma}{A.\arabic{lemma}}
\begin{lemma}(Sub-Gaussian tail condition, \citep{ravikumar2011high})\label{app: tail bound1}
Consider a zero-mean random vector $(Y_{1},\ldots,Y_{p})$ with covariance $\Sigma^{*}$ such that each $Y_{i}/\sqrt{\Sigma^{*}_{ii}}$ is sub-Gaussian with parameter $\sigma$. Given $n$ i.i.d samples, the sample covariance matrix $S$ satisfies the tail bound 
\begin{align}
    \mathbb{P}\left[\vert S_{ij} - \Sigma^{*}_{ij} \vert>\delta\right]\leq 4 \exp\Big\{-\frac{n\delta^{2}}{128(1+4\sigma^{2})^{2}\max\limits_{i}(\Sigma^{*}_{ii})^{2}}\Big\}, 
\end{align}
for all $\delta\in(0,8(1+4\sigma^{2})\max\limits_{i}(\Sigma^{*}_{ii}))$. 
\end{lemma}

Let $W_{ij}=S_{ij}-\Sigma^{*}_{ij}$, where $\Sigma^{*}={\Theta^*}^{-1}$. This difference quantity, which signifies the amount of noise in the data, plays a key role in bounding the error term $\Vert \widehat{B}-B^{*}\Vert_{\infty}$. We later show that  if $W_{ij}$ is small, then we can guarantee that our estimator $\widehat{B}$ is close to $B^{*}$ in the element-wise $\ellinf-$norm.

By taking a union bound over all entries of $\vert W_{ij}\vert$, from Lemma \ref{app: tail bound1}, it follows that 
\begin{align}\label{eq: controlnoise}
    \mathbb{P}\left[\Vert W\Vert_{\infty}\geq\delta_{f}(n,p^{\tau})\right]\leq \frac{p^{2}}{f(n,\delta_{f}(n,p^{\tau}))}=\frac{1}{p^{\tau-2}}, 
\end{align}
for some $\tau>2$. The above bound gives an explicit control on the noise term. 

We now state and prove our master lemma which gives support recovery guarantees and $\ellinf$ norm bounds for our estimator $\widehat{B}$ for distributions satisfying tail condition $\mathcal{T}(f,v_{*})$ in Definition \ref{def: TC}.

\begin{lemma}(Master lemma)\label{app: master lemma}
Consider a distribution satisfying the incoherence assumption with parameter $\alpha\in(0,1]$ and the tail condition $\mathcal{T}(f,v_{*})$. Let $\widehat{B}$ be the unique solution of the log-determinant problem in   \eqref{appeq: log-det problem} with $\lambda_{n} = 2\nu_{D^{2}}\nu_{B^{*}}\delta_{f}(n,p^{\tau})$ for some $\tau>2$. Then if the sample size is lower bounded as 
\begin{align}\label{appeq: lower bound}
    n>n_{f}(1/\max\{v_{*},24d\nu_{D^{2}}\nu_{B^{*}}\max\{\nu_{{\Gamma^{*}}^{-1}}\nu_{{B^{*}}^{-1}},2\nu^{2}_{{\Gamma^{*}}^{-1}}\nu^{3}_{{B^{*}}^{-1}},2\alpha^{-1}d^{-1}\} \},p^{\tau}), 
\end{align}
then with probability greater than $1-\frac{1}{p^{\tau-2}}$, the estimate $\widehat{B}$ recovers the sparsity structure of $B^{*}$ ie. ($\widehat{B}_{E^{c}} = B^{*}_{E^{c}}$). Furthermore $\widehat{B}$ satisfies the $\ellinf$ bound $\Vert \widehat{B}-B^{*}\Vert_{\infty}\leq 8\nu_{{\Gamma^{*}}^{-1}}\nu_{D^{2}}\nu_{B^{*}}\delta_{f}(n,p^{\tau})$. 
\end{lemma}
\begin{proof}
We first show that the Primal Dual Witness (PDW) construction (see sec \ref{sec: PDW}) succeeds with the probability stated in the lemma. This amounts to showing that the inequality in \eqref{eq: strict duality conditions} holds with the required probability. To this aim, let $\mathcal{A}$ denote the event that $\Vert W\Vert_{\infty}\leq
\delta_{f}(n,p^{\tau})$. We have previously shown in \eqref{eq: controlnoise} that $\mathbb{P}[\mathcal{A}]\geq 1-1/p^{\tau-2}$. Conditioned on the event $\mathcal{A}$, we show that the inequality in \eqref{eq: strict duality conditions} is satisfied. 

From Lemma \ref{app: control of delta}, we have 
\begin{align}
    r &= 4\nu_{{\Gamma^{*}}^{-1}}\left[\nu_{D^{2}}\nu_{B^{*}}\Vert W\Vert_{\infty}+0.5{\lambda_{n}}\right],
\end{align}
substituting for $\lambda_{n}=2\nu_{D^{2}}\nu_{B^{*}}\delta_{f}(n,p^{\tau})$ as given in the assumption, we get 
\begin{align}
  r & \leq 8\nu_{{\Gamma^{*}}^{-1}}\nu_{D^{2}}\nu_{B^{*}}\delta_{f}(n,p^{\tau})  
\end{align}
From assumption on the sample size $n$ in \eqref{appeq: lower bound} and the monotonicity property \eqref{appeq: monotonicity} we have $0.5\lambda_{n} = \nu_{D^{2}}\nu_{B^{*}}\delta_{f}(n,p^{\tau})\leq \alpha/48$, which implies that $\lambda_{n}<1$. Similarly from \eqref{appeq: monotonicity} and \eqref{appeq: lower bound} we have $r\leq 8\nu_{{\Gamma^{*}}^{-1}}\nu_{D^{2}}\nu_{B^{*}}\delta_{f}(n,p^{\tau})\leq \min\{1/(3d\nu_{{B^{*}}^{-1}}),1(6d\nu_{{\Gamma^{*}}^{-1}}\nu^{3}_{{B^{*}}^{-1}})\}$. Therefore the assumption in Lemma \ref{app: control of delta} is satisfied resulting in
\begin{align}\label{eq: satisfylemma4}
    \Vert \Delta\Vert_{\infty}\leq r\leq \min\left[\frac{1}{3d\nu_{{B^{*}}^{-1}}},\frac{1}{6d\nu_{{\Gamma^{*}}^{-1}}\nu^{3}_{{B^{*}}^{-1}}}\right]
\end{align}
Define $\delta_{f}\triangleq \delta_{f}(n,p^{\tau})$. We show that the every component in the max term of \eqref{eq: strict duality conditions} are bounded by $\alpha \lambda_n/24$. We begin with the first component: 
\begin{align}
   \vertiii{\Gamma(D^{2}\Delta)+\Gamma(D^{2}B^{*})}_{\infty}\Vert W\Vert_{\infty}&\leq \left[\vertiii{D^{2}\Delta + D^{2}B^{*}}_{\infty}\right]\delta_{f}\\
   &\leq \left[\vertiii{D^{2}\Delta}_{\infty}+\vertiii{D^{2}B^{*}}_{\infty}\right]\delta_{f}\\
   &\leq \left[\nu_{D^{2}}d\Vert \Delta\Vert_{\infty}+\nu_{D^{2}}\nu_{B^{*}}\right]\frac{\alpha}{48\nu_{D^{2}}\nu_{B^{*}}}\\
   &\overset{(a)}{\leq}\left[1+\frac{1}{3\nu_{B^{*}}\nu_{{B^{*}}^{-1}}}\right]\frac{\alpha}{48}\\
   &\overset{(b)}{\leq} \frac{\alpha}{36}\leq \frac{\alpha}{24}\overset{(c)}{\leq}\frac{\alpha\lambda_{n}}{24}, 
\end{align}
where (a) follows from \eqref{eq: satisfylemma4}; (b) follows because $\nu_{B^{*}}\nu_{{B^{*}}^{-1}}\geq 1$ ;and (c) follows because $\lambda_{n}<1$. 

We show the second component $\Vert R(\Delta)\Vert_{\infty}\leq \alpha\lambda_{n}/24$. In fact, 
\begin{align}
    \Vert R(\Delta)\Vert_{\infty}&\overset{(a)}{\leq} \frac{3}{2}d\Vert \Delta\Vert_{\infty}^{2}\nu^{3}_{{B^{*}}^{-1}}\\
    &\overset{(b)}{\leq} \frac{3}{2}dr\nu^{3}_{{B^{*}}^{-1}}r\\
    &\overset{(c)}{\leq}\frac{3}{2}d\left[\frac{1}{6d\nu_{{\Gamma^{*}}^{-1}}\nu^{3}_{{B^{*}}^{-1}}}\right]\nu^{3}_{{B^{*}}^{-1}}(8\nu_{{\Gamma^{*}}^{-1}}\nu_{D^{2}}\nu_{B^{*}}\delta_{f})\\
    &= 2\nu_{D^{2}}\nu_{B^{*}}\delta_{f}\leq \frac{\alpha}{24}\leq \frac{\alpha\lambda_{n}}{24}, 
\end{align}
where (a) holds because, as shown in \eqref{eq: satisfylemma4}, $\Vert \Delta\Vert_{\infty}$ satisfies the assumption in Lemma \ref{app: control remainder}; (b) holds because $\Vert \Delta\Vert_{\infty}\leq r$; and (c) is a consequence of the inequality in \eqref{eq: satisfylemma4}. 

We show that the third component $\vertiii{\Gamma(D^{2}\Delta)}_{\infty}\Vert {\Theta^{*}}^{-1}\Vert_{\infty}\leq \alpha\lambda_{n}/24$. In fact, 
\begin{align}
    \vertiii{\Gamma(D^{2}\Delta)}_{\infty}\Vert {\Theta^{*}}^{-1}\Vert_{\infty} &= \vertiii{D^{2}\Delta}_{\infty}\Vert {\Theta^{*}}^{-1}\Vert_{\infty}\\
    &\leq \nu_{D^{2}}d\Vert {\Theta^{*}}^{-1}\Vert_{\infty}\Vert \Delta\Vert_{\infty}\\
    &\overset{(a)}{\leq} \nu_{D^{2}}d\Vert {\Theta^{*}}^{-1}\Vert_{\infty}r\\
    &\overset{(b)}{\leq} d\nu_{D^{2}}\left[\frac{1}{4d\nu_{D^{2}}\nu_{{\Gamma^{*}}^{-1}}}\right]\left[8\nu_{{\Gamma^{*}}^{-1}}\nu_{D^{2}}\nu_{B^{*}}\delta_{f}\right]\\
    &\leq 2\nu_{D^{2}}\nu_{B^{*}}\left[\frac{\alpha}{48\nu^{D^{2}}\nu_{B^{*}}}\right]=\frac{\alpha}{24}\leq\frac{\alpha\lambda_{n}}{24}, 
\end{align}
where (a) holds because $\Vert \Delta\Vert_{\infty}\leq r$ and (b) follows by invoking the assumption in \eqref{appeq: Hessian regularity}. Since the sufficient conditions for strict dual feasibility are satisfied, the PDW construction succeeds. Therefore $\widehat{B}=\widetilde{B}$. Since by definition $\widetilde{B}_{E} = B^{*}_{E^{c}}=0$, the estimator $\widehat{B}$ recovers the sparsity structure of $B^{*}$. Now, since $\Delta = \widehat{B}-B^{*}$ and $\Vert \Delta\Vert_{\infty}\leq 8\nu_{{\Gamma^{*}}^{-1}}\nu_{D^{2}}\nu_{B^{*}}\delta_{f}(n,p^{\tau})$, we have $\Vert\widehat{B}-\widetilde{B}\Vert_{\infty}\leq 8\nu_{{\Gamma^{*}}^{-1}}\nu_{D^{2}}\nu_{B^{*}}\delta_{f}(n,p^{\tau})$. 
\end{proof}

We use Lemma \ref{app: tail bound1} and Lemma \ref{app: master lemma} to prove our main result for sub-gaussian distributions.
\setcounter{theorem}{0}
\begin{theorem}{(Support recovery: Sub-Gaussian)}\label{app: sub-gaussian support recovery} Let $Y=(Y_1,\ldots,Y_p)$ be the node potential vector. Suppose that $Y_{i}/\sqrt{\Sigma^{*}_{ii}}$ is sub-Gaussian with parameter $\sigma$  and assumptions [\textbf{A1-A3}] hold. Let the regularization parameter $\lambda_{n} = C_{0}\sqrt{\tau(\log 4p)/n}$, where $C_0$ is given below. If the sample size 
$n>C^2_{1}d^{2}(\tau\log p+\log 4)$, the following hold
with probability at least $1-\frac{1}{p^{\tau-2}}$, for some $\tau>2$:
\begin{enumerate}[label=(\alph*)]
\item  $\widehat{B}$ exactly recovers the sparsity structure of $B^{*}$; that is, $\widehat{B}_{E^{c}} = 0$,
\item \label{eq: ellinf consistency}$\widehat{B}$ satisfies the element-wise $\ell_{\infty}$ bound $\Vert \widehat{B} - B^{*}\Vert_{\infty}\leq C_{2}\sqrt{\frac{\tau\log p+\log 4}{n}}$, and
\item $\widehat{B}$ satisfies sign consistency if $\vert B^{*}_{\min} \vert\geq 2C_{2}\sqrt{\frac{\tau\log p + 4}{n}}$, $B^{*}_{\min}\triangleq \min_{(i,j)\in \mathcal{E}(B^{*})}\vert B^{*}_{ij}\vert$, 
\end{enumerate}
where $C_{1}=192\sqrt{2}\left[(1+4\sigma^{2})\max\limits_{i}(\Sigma^{*}_{ii})\nu_{D^{2}}\nu_{B^{*}}\right]\max\{\nu_{{\Gamma^{*}}^{-1}}\nu_{{B^{*}}^{-1}},2\nu^{2}_{{\Gamma^{*}}^{-1}}\nu^{3}_{{B^{*}}^{-1}},2\alpha^{-1}d^{-1}\}$, $C_{2} = [64\sqrt{2}(1+4\sigma^{2})\max\limits_{i}(\Sigma^{*}_{ii})\nu_{{\Gamma^{*}}^{-1}}\nu_{D^{2}}\nu_{B^{*}}]$, and $C_{0} = C_{2}/(4\nu_{{\Gamma^{*}}^{-1}})$. 
\end{theorem}
\begin{proof} Part (a): From Lemma \ref{app: master lemma}, we have that if $n>n_{f}(\delta,p^{\tau})$, then $\widehat{B}$ recovers the exact sparsity structure of $B^{*}$. We compute $n_{f}(\delta,p^{\tau})$. 
Using the tail bound for sub-gaussian distributions (see Lemma \ref{app: tail bound1}), the decay function $f(n,\delta) = \frac{1}{4}\exp\Big\{\frac{n\delta^{2}}{c_{*}}\Big\}$, where $c_{*} = 128(1+4\sigma^{2})^{2}\max\limits_{i}(\Sigma^{*}_{ii})^{2}$. From the definition of inverse function and monotonicity of $f(n,\delta)$ in \ref{def: TC}, we have $n_{f}(\delta,p^{\tau}) = \frac{c_{*}\log(4p^{\tau})}{\delta^{2}}$. Substituting for $\delta$ from Lemma \ref{app: master lemma}, we get 
\begin{align}
  n_{f}(\delta,p^{\tau}) =  C_{1}^{2} d^{2}(\tau\log p+\log 4). 
\end{align}
Therefore, from Lemma \ref{app: master lemma}, if $n>C_{1}^{2} d^{2}(\tau\log p+\log 4)$, the estimator $\widehat{B}$ recovers the sparsity structure of $B^{*}$. 

Part(b): From Lemma \ref{app: tail bound1} we compute $\delta$.  Using the monotonicity property of $f(n,\delta)$ and setting 
\begin{align}
\delta\triangleq \delta_{f}(n,p^{\tau}) = \sqrt{\frac{c_{*}\log(4p^{\tau})}{n}} = 8\sqrt{2}(1+4\sigma^{2})\max\limits_{i}(\Sigma^{*}_{ii}) \sqrt{\frac{\tau\log p+\log 4}{n}}. 
\end{align}
Also we have from Lemma \ref{app: master lemma} that $\Vert \widehat{B} - B^{*}\Vert_{\infty}\leq 8\nu_{{\Gamma^{*}}^{-1}}\nu_{D^{2}}\nu_{B^{*}}\delta_{f}(n,p^{\tau})$. Thus, 
\begin{align}
    \Vert \widehat{B} - B^{*}\Vert_{\infty}\leq\underbrace{ 64\sqrt{2}(1+4\sigma^{2})\max\limits_{i}(\Sigma^{*}_{ii})\nu_{{\Gamma^{*}}^{-1}}\nu_{D^{2}}\nu_{B^{*}}}_{C_2}\sqrt{\frac{\tau\log p+\log 4}{n}}. 
\end{align}\\
Part(c): We prove the sign consistency of $ \widehat{B}$ by contradiction. Let $\vert B^{*}_{\min} \vert\geq 2C_{2}\sqrt{\frac{\tau\log p + 4}{n}}$ be as in the theorem's hypothesis. Suppose that $\text{sign}(\widehat{B})\neq \text{sign}(B^{*})$. Then, by elementary algebra we have that $\Vert \widehat{B}\!-\!B^{*}\Vert_{\infty}\!>\! 2C_{2}\sqrt{\frac{\tau\log p + 4}{n}}$. This contradicts the bound in part (b). Thus, $\text{sign}(\widehat{B})= \text{sign}(B^{*})$.
\end{proof}
We now show Frobenius and spectral norm consistency for the sub-gaussian distribution. Recall that $\mathcal{E}(B^{*}) = \{(i,j): B^{*}_{ij}\neq 0, \text{for all} \hspace{3px} i\neq j\}$ is the edge set of $B^{*}$. Thus, $s = \vert\mathcal{E}(B^{*})\vert$ is the number of non-zero off-diagonal elements in $B^{*}$. 
\setcounter{corollary}{0}
\begin{corollary}\label{app: corollary1}
Let $s = \vert\mathcal{E}(B^{*})\vert$ be the cardinality of $\mathcal{E}(B^{*})$. Under the same hypotheses in Theorem \ref{app: sub-gaussian support recovery}, with probability greater than $1-\frac{1}{p^{\tau-2}}$, the estimator $\widehat{B}$ satisfies 
\begin{align}\label{eq: F and S bound sub-Gauss}
    \Vert \widehat{B} - B^{*}\Vert_{F} &\leq C_{2}\sqrt{\frac{(s+p)(\tau\log p + 4)}{n}} \,\, \text{ and }\,\,
    \Vert{\widehat{B} - B^{*}}\Vert_{2} \leq C_{2}\min\{d,\sqrt{s+p}\}\sqrt{\frac{\tau\log p + 4}{n}}. 
\end{align}
\end{corollary}
\begin{proof}
Consider the following inequality: 
\begin{align}
    \Vert \widehat{B}-B^{*}\Vert_{F}^2= \sum_{i,j}\left(\widehat{B}_{ij}-B^{*}_{ij}\right)^{2}&= \sum_{i}\left(\widehat{B}_{ii}-B^{*}_{ii}\right)^{2}+\sum_{i\neq j}\left(\widehat{B}_{ij}-B^{*}_{ij}\right)^{2}\\
    &\leq p\Vert \widehat{B}-B^{*}\Vert^{2}_{\infty} + s\Vert \widehat{B}-B^{*}\Vert^{2}_{\infty}\\
    &=(s+p)\Vert \widehat{B}-B^{*}\Vert_{\infty}^2, 
\end{align}
where the inequality follows because there are at most $p$ non-zero diagonal terms and $s$ non-zero off-diagonal terms in $\widehat{B}-B^*$. The latter fact is a consequence of Theorem \ref{app: sub-gaussian support recovery} (a), which ensures that $\widehat{B}_{E^{c}} = B^{*}_{E^{c}}$ with high probability when $n=\Omega(d^{2}\log p)$. We obtain the Frobenius norm bound in \eqref{eq: F and S bound sub-Gauss} by upper bounding
$\Vert \widehat{B}-B^{*}\Vert_{\infty}$ using the result in Theorem \ref{app: sub-gaussian support recovery} (b). 
We now show spectral norm consistency. From matrix norm equivalence conditions \citep{horn2012matrix}, we have 
\begin{align}
    \Vert \widehat{B}-B^{*}\Vert_{2}\leq \vertiii{\widehat{B}-B^{*}}_{\infty}
    \leq d\Vert \widehat{B}-B^{*}\Vert_{\infty}\label{eq: matrix norm equivalence1}
\end{align}
and that
\begin{align}\label{eq: matrix norm equivalence2}
    \Vert \widehat{B}-B^{*}\Vert_{2}\leq \Vert \widehat{B}-B^{*}\Vert_{F}\leq \sqrt{s+p}\Vert \widehat{B}-B^{*}\Vert_{\infty}. 
\end{align}
These two bounds can be unified into one single bound as 
\begin{align}
    \Vert \widehat{B}-B^{*}\Vert_{2}\leq \min\{\sqrt{s+p},d\}\Vert \widehat{B}-B^{*}\Vert_{\infty}. 
\end{align}
This concludes the proof. 
\end{proof}
Next we prove our second main result for random vectors with bounded moments. We need the following standard concentration inequality result. 
\begin{lemma}(Tail bounds for random variables with bounded moments, \citep{ravikumar2011high})\label{app: tail boound2}
For a random vector $(Y_{1},\ldots,Y_{p})$, suppose there exists a positive integer $k$ and scalar $M_{k}\in \mathbb{R}$ with
\begin{align}\label{appeq: Bounded moments}
    \mathbb{E}\left[\frac{Y_{i}}{\sqrt{\Sigma^{*}_{ii}}}\right]^{4k}\leq M_{k}. 
\end{align}
Given $n$ i.i.d samples, the sample covariance matrix $S$ admits the following concentration inequality 
\begin{align}
    \mathbb{P}\left[\vert S_{ij}-\Sigma^{*}_{ij}\vert>\delta\right]\leq \frac{2^{2k}(\max_{i}\Sigma^{*}_{ii})^{2k}C_{k}(M_{k}+1)}{n^{k}\delta^{2k}}.  
\end{align}
where $C_{k}\geq 0$ is a constant depending only on $k$. 
\end{lemma}

\begin{theorem}(Support Recovery: Bounded Moments)\label{app: bounded moment support recovery} Let $Y=(Y_1,\ldots,Y_p)$ be the node potential vector. Suppose that $Y_{i}/\sqrt{\Sigma^{*}_{ii}}$ has bounded moment as in \eqref{appeq: Bounded moments} and assumptions [\textbf{A1-A3}] hold. Let the regularization parameter $\lambda_{n}= C_{0}\sqrt{\tau(\log 4p)/n}$, with $C_{0}$ defined in Theorem \ref{app: sub-gaussian support recovery}. If the sample size $n > C_{4}d^{2}p^{\tau/k}$, then with probability more than $1-{1}/{p^{\tau-2}}$, for some $\tau>2$, the following hold:
\begin{enumerate}[label=(\alph*)]
\item  $\widehat{B}$ exactly recovers the sparsity structure of $B^{*}$; that is, $\widehat{B}_{E^{c}} = 0$,
\item \label{eq: ellinf consistency}$\widehat{B}$ satisfies the element-wise $\ell_{\infty}$ bound $\Vert \widehat{B} - B^{*}\Vert_{\infty}\leq C_{5}\sqrt{\frac{p^{\tau/k}}{n}}$, and
\item $\widehat{B}$ satisfies sign consistency if $\vert B^{*}_{\min} \vert\geq 2C_{5}\sqrt{\frac{p^{\tau/k}}{n}}$, 
\end{enumerate}
where $C_{4} =\left[48(\max\limits_{i}\Sigma^{*}_{ii})\left(C_{k}(M_{k}+1)\right)^{1/2k}\nu_{D^{2}}\nu_{B^{*}}\max\{\nu_{{\Gamma^{*}}^{-1}}\nu_{{B^{*}}^{-1}},2\nu^{2}_{{\Gamma^{*}}^{-1}}\nu^{3}_{{B^{*}}^{-1}},2\alpha^{-1}d^{-1}\}\right]^{2}$, $C_{5}=16(\max_{i}\Sigma^{*}_{ii})\left(C_{k}(M_{k}+1)\right)^{1/2k}\nu_{{\Gamma^{*}}^{-1}}\nu_{D^{2}}\nu_{B^{*}}$.
\begin{proof}
The proof follows along the same lines of Theorem \ref{app: sub-gaussian support recovery}. Hence, to avoid redundancy, we provide only high-level details. Part (a) We use the polynomial type tail bound in \ref{app: tail boound2} to compute $n_{f}(\delta,p^{\tau})$, we therefore have $n_{f}(\delta,p^{\tau}) = \frac{(c_{*}p^{\tau})^{1/k}}{\delta^{2}}$ and substituting for $c^{*}$ and $\delta$ as given in Lemma \ref{app: tail boound2} and Lemma \ref{app: master lemma} respectively, we get 
\begin{align}
  n_{f}(\delta,p^{\tau}) =  C_{4} d^{2}p^{\tau/k}. 
\end{align}
Part(b): From Lemma \ref{app: tail boound2}, we have $f(n,\delta) = \frac{n^{k}\delta^{2k}}{c_{*}}$, where $c_{*} = 2^{2k}(\max\limits_{i}\Sigma^{*}_{ii})^{2k}C_{k}(M_{k}+1)$. Thus setting
\begin{align}
       \delta = \delta_{f}(n,p^{\tau})=\left(\frac{c_{*}p^{\tau}}{n}\right)^{1/2k} = 2(\max_{i}\Sigma_{ii}^{*})(C_{k}(M_{k}+1))^{1/2k}\sqrt{\frac{p^{\tau/k}}{n}}. 
\end{align}
On the other hand, from Lemma \ref{app: master lemma}, we have $\Vert \widehat{B} - B^{*}\Vert_{\infty}\leq 8\nu_{{\Gamma^{*}}^{-1}}\nu_{D^{2}}\nu_{B^{*}}\delta_{f}(n,p^{\tau})$. Thus, \begin{align}
    \Vert \widehat{B} - B^{*}\Vert_{\infty}\leq 16(\max_{i}\Sigma_{ii}^{*})(C_{k}(M_{k}+1))^{1/2k}\nu_{{\Gamma^{*}}^{-1}}\nu_{D^{2}}\nu_{B^{*}}\sqrt{\frac{p^{\tau/k}}{n}}. 
\end{align}
Part (c): similar to the contradiction argument in Theorem \ref{app: sub-gaussian support recovery}. Details are omitted. 
\end{proof}
\end{theorem}
We present Frobenius and spectral norm consistency results for distributions with bounded moments.
\begin{corollary}\label{app: corollary2}
Suppose the hypotheses in Theorem \ref{app: bounded moment support recovery} hold. Then with probability greater than $1-\frac{1}{p^{\tau-2}}$:  $\Vert \widehat{B} - B^{*}\Vert_{F} \leq C_{5}\sqrt{\frac{(s+p)(p^{\tau/k})}{n}}$ and $
    \Vert{\widehat{B} - B^{*}}\Vert_{2}\leq C_{5}\min\{d,\sqrt{s+p}\}\sqrt{\frac{p^{\tau/k}}{n}}$, where $s = \vert\mathcal{E}(B^{*})\vert$.
    \end{corollary}
\begin{proof}
The proof follows along the same lines of Corollary \ref{app: corollary1}. Hence, the details are omitted. 
\end{proof}

%%%%%%%%%%%%%%%%%%%%%%%%%%%%%%%%%%%%%%%%%%%%%%%%%%%%%%%%%%%%
\end{document}